\documentclass[11pt]{article}

\usepackage{acl}

\usepackage{times}
\usepackage{latexsym}
\usepackage[T1]{fontenc}
\usepackage[utf8]{inputenc}
\usepackage{microtype}
\usepackage{inconsolata}
\usepackage{graphicx}
\usepackage{cuted}
\usepackage{caption}
\usepackage[percent]{overpic}

\usepackage{booktabs}
\usepackage{amsfonts}
\usepackage{amsmath}
\usepackage{amssymb}
\usepackage{algorithm}
\usepackage{algorithmic}
\usepackage{multirow}
\usepackage{array}
\usepackage{enumitem}
\usepackage{subcaption}
\usepackage{float}
\usepackage{amsthm}
\usepackage{placeins}
\usepackage[table]{xcolor}
\usepackage{colortbl}
\usepackage{soul}
\usepackage{nicefrac}

\newtheorem{proposition}{Proposition}
\newtheorem{remark}{Remark}

\title{TOPS: First-Principles Visual Token Pruning via Constructing Token Optimal Preservation Sets for Efficient MLLM Inference}

\vspace{22pt}
\author{
{\small\textbf{
Tinghao Wang$^{1,2,*}$,
Yichen Guo$^{1,3,*}$,
Rui Huang$^{2,*,\dagger}$,
Junchi Yao$^{4}$,
Qizhe Zhang$^{1}$,
}}\\
{\small\textbf{
Zheng Lu$^{2}$,
Yuan Zhang$^{1}$,
Jiajun Cao$^{1}$,
Zhirong Shen$^{2}$,
Yaosong Du$^{2}$,
Lijie Hu$^{4}$,
Shanghang Zhang$^{1,\ddagger}$
}}\\[0.3em]
{\small $^{1}$State Key Laboratory of Multimedia Information Processing,}\\[-0.4em]
{\small School of Computer Science, Peking University}\\
{\small $^{2}$University of Electronic Science and Technology of China}\\
{\small $^{3}$Nanyang Technological University,
$^{4}$MBZUAI}
}
\makeatletter
\newcommand{\TopWideTable}[1]{%
\begin{strip}
\centering
\def\@captype{table}%
#1
\end{strip}
}
\makeatother
\makeatletter
\AtBeginDocument{
\let\@oldmaketitle\@maketitle
\renewcommand{\@maketitle}{%
  \@oldmaketitle
  \vspace{-10pt}

  \noindent
  \begin{minipage}{\textwidth}
  \centering

  \noindent
  \begin{minipage}{0.98\textwidth}
  \begin{minipage}[t]{1.8em}
  \vspace{2mm}
  (a)
  \end{minipage}%
  \hspace{0.2em}%
  \begin{minipage}[t]{0.94\textwidth}
  \vspace{0pt}
  \centering
  \includegraphics[width=\linewidth]{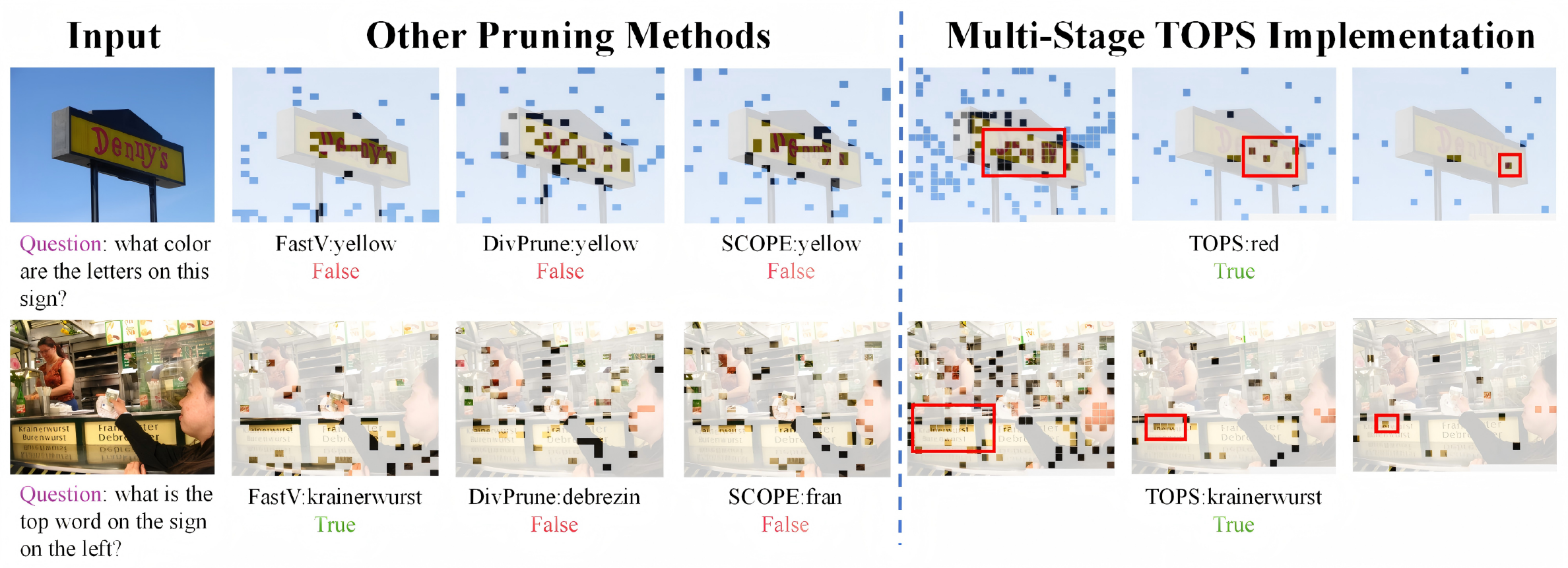}
  \end{minipage}
  \end{minipage}
  \vspace{-3mm}

  \noindent
  \begin{minipage}{0.98\textwidth}
  \begin{minipage}[t]{1.8em}
  \vspace{2mm}
  (b)
  \end{minipage}%
  \hspace{0.2em}%
  \begin{minipage}[t]{0.94\textwidth}
  \vspace{0pt}
  \centering
  \includegraphics[width=\linewidth]{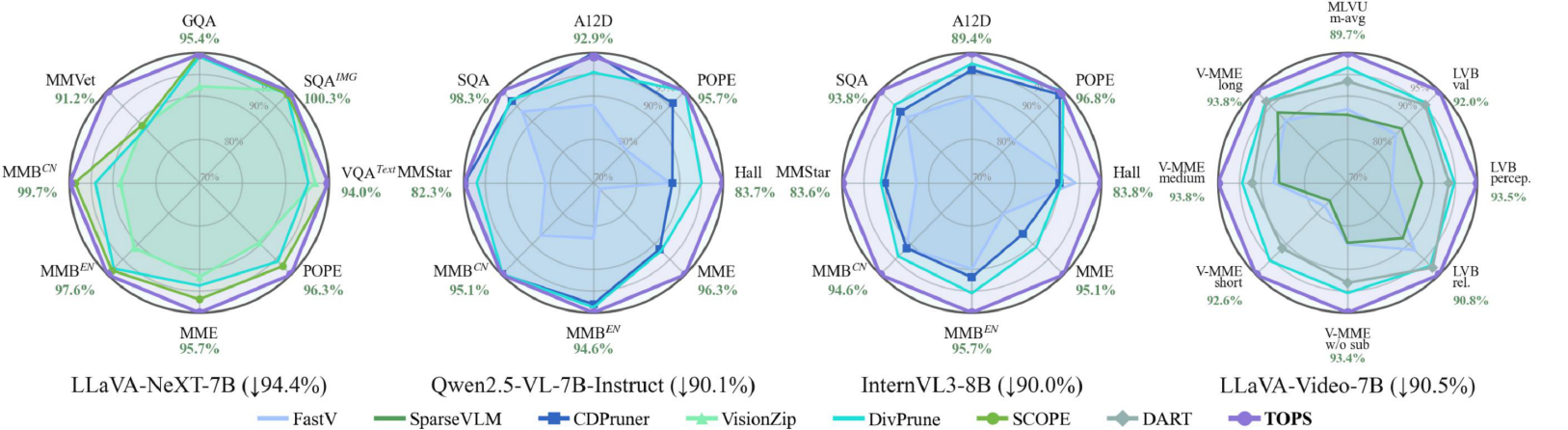}
  \end{minipage}
  \end{minipage}

  \vspace{-7pt}
  \captionof{figure}{
  (a) \textbf{Qualitative comparison of pruning methods.}
  On detail-sensitive VQA questions, single-criterion pruning methods, including attention-based, diversity-based, and coverage-based methods, often fail to answer, whereas the multi-stage TOPS module helps model preserve key visual evidence and produce the correct answers.
  (b) \textbf{Performance comparison on four mainstream MLLMs.}
  We validate TOPS across four architectures. TOPS consistently covers the largest area, demonstrating superior performance across all models and benchmarks.
  }
  \label{fig:intro_overview}

  \end{minipage}
  \vspace{3pt}
}
}
\makeatother



\begin{document}
\maketitle
\makeatletter
\begingroup
\renewcommand{\thefootnote}{}
\renewcommand{\@makefntext}[1]{\noindent#1}
\footnotetext{\textsuperscript{*}Equal contribution. \textsuperscript{$\dagger$}Project leader. \textsuperscript{$\ddagger$}Corresponding author.}
\endgroup
\makeatother

\vspace{10pt}
\begin{abstract}
\vspace{-8pt}
Multimodal large language models (MLLMs) have achieved strong multimodal reasoning capabilities, but their efficiency is limited by the large number of visual tokens, which introduces substantial computational overhead. Visual token pruning offers a natural solution, yet existing methods are imperfect: attention-based criteria tend to retain redundant tokens, while diversity-based criteria are often agnostic to user instructions. Even methods that combine multiple criteria still lack a principled formulation of the intrinsic objective of token pruning. In this paper, we revisit visual token pruning from a first-principles perspective and formulate it as constructing Token Optimal Preservation Sets. Through a top-down information-theoretic analysis, we identify three fundamental principles for effective token selection: Task Relevance, Information Coverage, and Semantic Diversity. Based on these principles, we propose \textbf{TOPS}, a training-free and model-agnostic pruning module that can be applied to various MLLMs. Extensive experiments on 7 MLLM backbones and 14 benchmarks demonstrate that TOPS outperforms prior methods under diverse pruning settings. Notably, on LLaVA-NeXT, TOPS removes 77.8\% of visual tokens while preserving 100.0\% and 100.6\% performance on its 7B and 13B models, respectively, suggesting that pruning redundant visual tokens can sometimes mitigate hallucination and inspire future lightweight MLLM design. Our code is available at \url{https://github.com/MAX-7237/TOPS}.
\end{abstract}
\vspace{-9pt}
\section{Introduction}
\vspace{-8pt}

Large language models (LLMs)~\cite{achiam2023gpt, hurst2024gpt,singh2025openai,yang2025qwen3,team2026qwen3,touvron2023llama,touvron2023llama2,grattafiori2024llama, team2023gemini, team2024gemini, comanici2025gemini,team2026kimi,team2025kimi} have achieved remarkable success in language understanding and reasoning. Building on these capabilities, multimodal large language models (MLLMs)~\cite{liu2023visual,liu2024llavanext,li2024llava,zhang2024llava,Qwen2.5-VL,bai2025qwen3,chen2024internvl,zhu2025internvl3} have made rapid progress in multimodal understanding. However, their efficiency is limited by numerous visual tokens, which are processed through all transformer layers~\cite{chen2024image,zhang2024sparsevlm,zhang2025vispruner,zhang2025vscan}. Since self-attention scales quadratically with sequence length, these tokens introduce substantial computational and memory overhead, especially for multi-image and high-resolution inputs. Therefore, reducing visual tokens while preserving performance is a critical challenge.

Previous methods~\cite{chen2024image, yang2025visionzip,wang2026entropyprune,cao2026fastdrivevla} have attempted to reduce visual tokens to lower the inference cost of MLLMs. Existing pruning approaches can be broadly categorized into three types. Attention-based methods~\cite{zhang2025vispruner, xing2024pyramiddrop, zhang2024sparsevlm} identify token importance via cross-modal attention or \texttt{cls} token attention, but often retain highly similar tokens, resulting in redundancy. Diversity-based methods~\cite{alvar2025divprune, wen2025stop} encourage semantic dispersion, yet are typically agnostic to user instructions and may discard task-critical evidence. Other methods combine multiple criteria or incorporate coverage-based objectives~\cite{song2025less, shang2025llava, zhang2025vispruner, baek2026agilepruner, zhang2025towards} to model the representativeness of selected subsets. However, despite these advances, existing methods mostly treat token pruning as a scoring problem and rank tokens based on heuristic criteria, without a principled justification for why such criteria are appropriate or sufficient for constructing an optimal token subset. Fundamentally, these approaches do not start from the intrinsic objective of pruning, but rely on heuristic scoring schemes.

To address these challenges, we move beyond conventional heuristic scoring schemes~\cite{rao2021dynamicvit,liang2022not,bolya2022token,chen2024image,shang2025llava} and revisit token pruning from a first-principles perspective. Instead of designing new scoring heuristics, we rethink the core objective of token selection, conduct a top-down analysis using information theory, and identify three fundamental principles for effective pruning---\textbf{Task Relevance}, \textbf{Information Coverage}, and \textbf{Semantic Diversity}. Based on these principles, we formulate token pruning as an optimal subset selection problem and propose TOPS, which constructs a compact yet sufficient token subset that satisfies the proposed properties and can be applied at any pruning point during MLLM inference. We further implement TOPS as a two-stage pipeline for fine-grained token reduction, where Stage I removes coarse visual redundancy and Stage II performs text-aware refinement. As shown in Figure.~\ref{fig:intro_overview}(a), TOPS better preserves task-critical visual evidence under aggressive pruning.

As a simple yet effective solution, TOPS does not depend on any specific visual encoder or language model, which means it can be readily implemented across any token-based MLLM. Extensive experiments across various MLLMs demonstrate the effectiveness and efficiency of TOPS, surpassing existing methods (Figure.~\ref{fig:intro_overview}(b)). For instance, on LLaVA-v1.5-7B~\cite{liu2023visual} and LLaVA-NeXT-7B~\cite{liu2024llavanext}, TOPS removes nearly 90\% of visual tokens while retaining 97.1\% and 99.1\% of the original performance.

Overall, our main contributions are as follows:
\begin{itemize}[leftmargin=1.2em, itemsep=2pt, topsep=2pt, parsep=0pt, partopsep=0pt]
    \item We revisit token pruning from a first-principles perspective, conduct a top-down analysis using information theory and identify three criteria that govern effective token selection.
    \item We propose \textbf{TOPS}, a training-free and model-agnostic pruning module, considering the fundamental criteria and dynamically constructing token optimal preservation sets in MLLMs.
    \item We conduct extensive experiments across various MLLMs and benchmarks, demonstrating TOPS consistently achieves state-of-the-art performance across different reduction ratios.
\end{itemize}
\vspace{-6pt}
\section{Related Work}
\begin{figure*}[t]
    \centering
    \includegraphics[width=\linewidth]{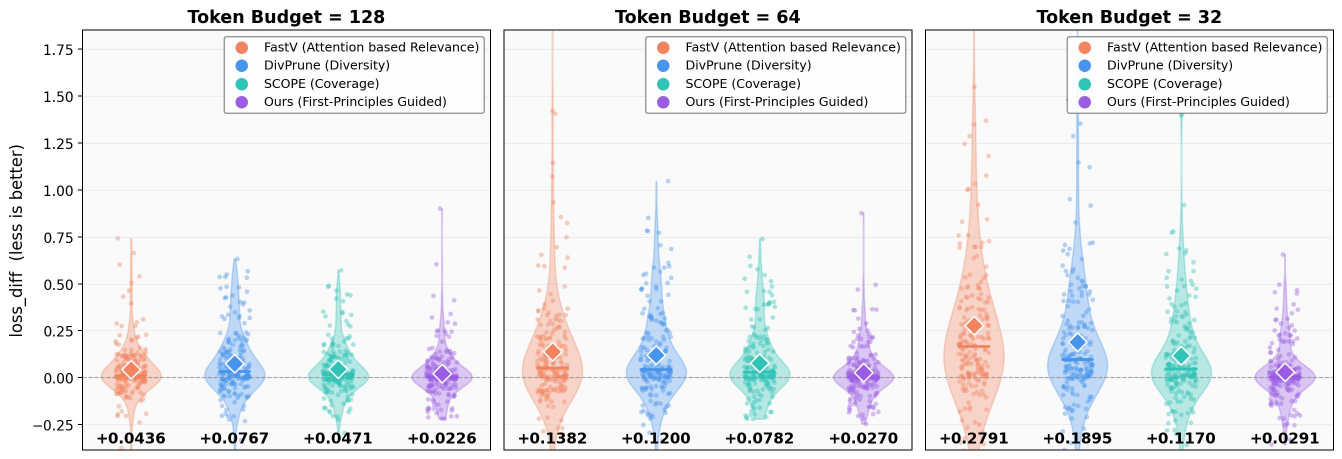}
    \caption{Logit fidelity of pruning methods across token budgets (128/64/32) on 200 MME samples. We report $\Delta\mathcal{L} = \mathcal{L}_{\text{pruned}} - \mathcal{L}_{\text{vanilla}}$, where lower values indicate smaller output distortion. TOPS consistently achieves the lowest loss increase, demonstrating stronger fidelity under aggressive pruning.}
    \label{fig:loss_diff_comparison_mme}
    \vspace{-4mm}
\end{figure*}

\vspace{-6pt}
\textbf{Multimodal large language models.}
Multimodal large language models (MLLMs)~\cite{liu2024improved, bai2023qwenvl, chen2024internvl, li2024llava, team2023gemini, hurst2024gpt, team2024gemini} extend large language models (LLMs)~\cite{brown2020language, achiam2023gpt, bai2023qwen, bai2025qwen2, touvron2023llama, peng2023instruction, bi2024deepseek} to multimodal understanding by encoding visual inputs as token sequences and processing them together with text tokens. However, visual tokenization introduces substantial computational overhead, since visual tokens are often far more numerous than text tokens and are propagated through all LLM layers. For example, LLaVA-1.5~\cite{liu2024improved} represents a $336\times336$ image with $576$ tokens, while LLaVA-NeXT~\cite{liu2024llavanext} can produce up to $2,880$ tokens for high-resolution inputs. The problem becomes more severe in video understanding~\cite{lin2024video, kondratyuk2023videopoet}, where long frame sequences lead to long visual token sequences and expensive inference. Therefore, effective token reduction is essential for scalable MLLM inference.
\vspace{3pt}

\textbf{Visual token reduction.}
Visual token reduction aims to improve MLLM efficiency by removing redundant visual tokens~\cite{jin2025efficient}. Existing training-free methods can be broadly categorized by their selection criteria. Attention-based methods estimate token importance from attention signals~\cite{chen2024image, xing2024pyramiddrop, zhang2024sparsevlm, zhangsparsevlm+, yang2025visionzip}, such as CLS-to-patch attention in VisionZip~\cite{yang2025visionzip} or text-guided cross-modal attention in FastV~\cite{chen2024image}. While effective for retaining salient tokens, they often preserve redundant tokens with similar semantics. Diversity-based methods~\cite{alvar2025divprune, wen2025stop} reduce redundancy by encouraging semantic dispersion, but are usually instruction-agnostic and may discard task-critical evidence. More recent methods combine importance, diversity, saliency, coverage, or progressive pruning strategies~\cite{song2025less, shang2025llava, zhang2025vispruner, zhang2025cdpruner, baek2026agilepruner, zhang2025towards, tan2026idpruner, wang2026entropyprune, deng2025scope, liu2024multi, zhang2025vscan}. Despite improved performance, they still rely on heuristic scoring schemes and lack a principled formulation of the intrinsic objective of token pruning.
\section{Motivation}
\label{sec:motivation}
\vspace{-1mm}
\begin{figure*}[t]
    \centering
    \includegraphics[width=\linewidth,trim=0 20pt 0 0,clip]{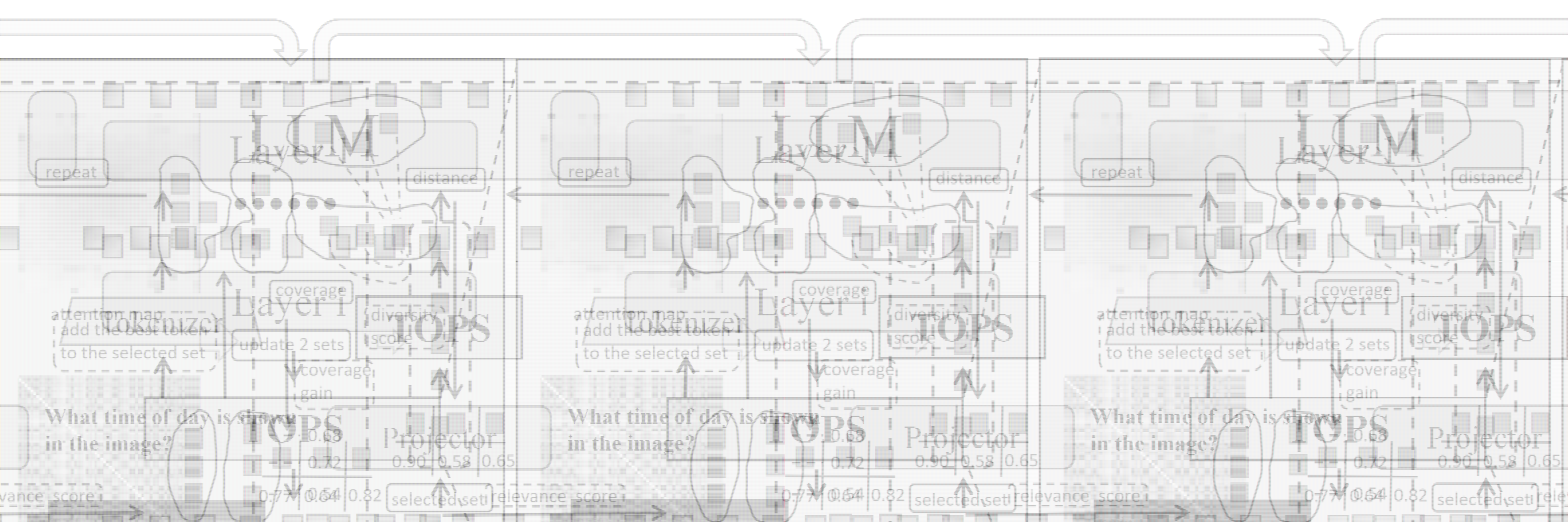}
    \vspace{-4mm}
    \caption{
        \textbf{Overview of TOPS.} Left: TOPS is a plug-and-play pruning module that can be applied at multiple stages during MLLM inference. Right: at each pruning point, TOPS constructs the optimal token preservation set by greedily selecting tokens that jointly maximize task relevance, information coverage, and semantic diversity---the three criteria derived from our first-principles formulation.
    }
    \label{fig:overview}
    \vspace{-4mm}
\end{figure*}
Modern MLLMs typically consist of a vision encoder $f_v$, a multimodal projector $g$, and a language model $f_{\phi}$. Given an image $X_v$ and a textual query $Q$, the model produces visual tokens $V = g(f_v(X_v)) \in \mathbb{R}^{n\times d}$. Since $n$ is typically much larger than the number of text tokens, visual token pruning seeks a subset 
$S^{*}=\arg\min_{S\subset V,\,|S|=K} D\!\left(f_{\phi}(S,Q)\,\|\,f_{\phi}(V,Q)\right)$,
where $D(\cdot \| \cdot)$ measures the output divergence between the pruned and full models.

\vspace{-2mm}
\subsection{First-Principles Pruning Formulation}
\label{sec:motivation_formulation}
\vspace{-1mm}

Stepping beyond the conventional visual token pruning paradigm, we revisit the problem from a first-principles perspective. Given the full visual token set $V$ and the textual query $Q$, the goal of token pruning is to retain a subset $S\subseteq V$ such that reasoning based on $(S,Q)$ remains consistent with that based on $(V,Q)$. We formalize this as an information-theoretic objective:
\begingroup
\setlength{\abovedisplayskip}{2pt}
\setlength{\belowdisplayskip}{2pt}
\setlength{\abovedisplayshortskip}{0pt}
\setlength{\belowdisplayshortskip}{2pt}
\begin{equation}
\mathop{\max}\nolimits_{\scriptstyle S\subseteq V,\,|S|\le K}
I(S;V,Q)
\label{eq:method_mi_obj}
\end{equation}
\endgroup
where $I(\cdot;\cdot)$ denotes mutual information. This defines the first-principles of visual token pruning: the optimal subset is one that maximally preserves information about both $V$ and $Q$.

\vspace{-2mm}
\subsection{Decomposition of the First-Principles}
\label{sec:motivation_decomposition}
\vspace{-1mm}

By looking inside the mutual information objective, we can decompose it via the chain rule:
\begingroup
\setlength{\abovedisplayskip}{2pt}
\setlength{\belowdisplayskip}{2pt}
\setlength{\abovedisplayshortskip}{0pt}
\setlength{\belowdisplayshortskip}{2pt}
\begin{align}
I(S;V,Q)
&=
\underbrace{I(S;Q)}_{\text{task relevance}}
+
\underbrace{I(S;V\mid Q)}_{\text{information coverage}}
\label{eq:method_chain}
\\[-0.4em]
&\mkern-18mu=
I(S;Q)+H(V\mid Q)-H(V\mid S,Q).
\nonumber
\end{align}
\endgroup
The two terms reflect two key properties of an optimal subset: \textbf{Task Relevance} $I(S;Q)$ measures how informative $S$ is for the query, while \textbf{Information Coverage} $I(S;V\mid Q)$ encourages $S$ to preserve sufficient information about the original visual set $V$ by minimizing $H(V\mid S,Q)$. Let $S_k=\{t_1,\ldots,t_k\}$ be the subset selected after $k$ steps. The total mutual information can be decomposed into per-token contributions:
\begingroup
\setlength{\abovedisplayskip}{2pt}
\setlength{\belowdisplayskip}{2pt}
\setlength{\abovedisplayshortskip}{0pt}
\setlength{\belowdisplayshortskip}{2pt}
\begin{equation}
I(S_k;V,Q)=\sum_{i=1}^{k} I(t_i;V,Q\mid S_{i-1}).
\label{eq:method_incremental}
\end{equation}
\endgroup
This means each token's value depends on what has already been selected. Specifically, the marginal contribution of adding token $t_i$ to the current set $S_{i-1}$ is the conditional mutual information:
\begingroup
\setlength{\abovedisplayskip}{2pt}
\setlength{\belowdisplayskip}{2pt}
\setlength{\abovedisplayshortskip}{0pt}
\setlength{\belowdisplayshortskip}{2pt}
\begin{equation}
\begin{aligned}
\Delta_i 
&=
I(S_i;V,Q)-I(S_{i-1};V,Q)
 \\
&=
I(t_i;V,Q\mid S_{i-1}).
\end{aligned}
\end{equation}
\endgroup
Furthermore, this per-step gain admits a natural upper bound governed by the conditional entropy of the candidate token:
\begingroup
\setlength{\abovedisplayskip}{2pt}
\setlength{\belowdisplayskip}{2pt}
\setlength{\abovedisplayshortskip}{0pt}
\setlength{\belowdisplayshortskip}{2pt}
\begin{equation}
0\le I(t_i;V,Q\mid S_{i-1})\le H(t_i\mid S_{i-1})\le H(t_i).
\end{equation}
\endgroup
The upper bound $H(t_i\mid S_{i-1})$ is tight when $t_i$ carries novel information beyond the current set, and diminishes as $t_i$ becomes redundant with already selected tokens. This reveals that, under a fixed budget, maximizing mutual information favors candidates that are semantically distinct from existing selections, yielding \textbf{Semantic Diversity} as a necessary condition for efficient and non-redundant subset construction.
\vspace{-2mm}

\begin{table*}[t]
    \centering
    \centering
\setlength{\tabcolsep}{2pt}
\caption{Comparison of pruning methods across LLaVA series. All numbers report \textbf{Rel.} (\%), the ratio of pruned model accuracy to baseline. \colorbox[rgb]{1.0,0.88,0.88}{Red}: attention-based. \colorbox[rgb]{0.88,0.96,0.88}{Green}: attention\&diversity. \colorbox[rgb]{0.88,0.93,1.0}{Blue}: diversity-based. \colorbox[rgb]{0.88,0.97,0.97}{Cyan}: coverage-based. \colorbox[rgb]{0.93,0.88,1.0}{Purple}: ours. ``--'': not available. Detailed per-benchmark results are provided in Appendix~\ref{sec:appendix_13b}.}
\vspace{-2mm}
\label{tab:llava_series_summary}
\resizebox{\textwidth}{!}{%
\begin{tabular}{l|ccc|ccc|ccc|ccc}
\toprule
\multirow{2}{*}{\textbf{Method}}
  & \multicolumn{3}{c|}{\textbf{LLaVA-1.5-7B} {\small(2023)}}
  & \multicolumn{3}{c|}{\textbf{LLaVA-1.5-13B} {\small(2023)}}
  & \multicolumn{3}{c|}{\textbf{LLaVA-NeXT-7B} {\small(2024a)}}
  & \multicolumn{3}{c}{\textbf{LLaVA-NeXT-13B} {\small(2024a)}} \\
  & \multicolumn{3}{c|}{576 tokens}
  & \multicolumn{3}{c|}{576 tokens}
  & \multicolumn{3}{c|}{Upper(Up.) 2880 tokens}
  & \multicolumn{3}{c}{Upper(Up.) 2880 tokens} \\
\cmidrule(lr){2-4}\cmidrule(lr){5-7}\cmidrule(lr){8-10}\cmidrule(lr){11-13}
Compress Ratio
  & $\downarrow$77.8\% & $\downarrow$88.9\% & $\downarrow$94.4\%
  & $\downarrow$77.8\% & $\downarrow$88.9\% & $\downarrow$94.4\%
  & $\downarrow$77.8\% & $\downarrow$88.9\% & $\downarrow$94.4\%
  & $\downarrow$77.8\% & $\downarrow$88.9\% & $\downarrow$94.4\% \\
Remain Token
  & 128       & 64        & 32
  & 128       & 64        & 32
  & Up.\ 640  & Up.\ 320  & Up.\ 160
  & Up.\ 640  & Up.\ 320  & Up.\ 160 \\
\midrule
\rowcolor[rgb]{1.0,0.88,0.88}
SparseVLM (ICML25) & 96.0\% & 86.2\% & --      & 98.2\% & 93.0\% & --      & 98.3\%  & 93.2\%  & --      & 99.7\%  & 96.4\%  & --      \\
\rowcolor[rgb]{0.88,0.96,0.88}
VisionZip (CVPR25) & 96.8\% & 93.0\% & 86.8\%  & 96.9\% & 93.2\% & 86.5\%  & 99.4\%  & 95.3\%  & 89.3\%  & 99.9\%  & 96.4\%  & 91.8\%  \\
\rowcolor[rgb]{0.88,0.93,1.0}
DivPrune (CVPR25)  & 96.7\% & 93.7\% & 90.2\%  & 96.8\% & 94.2\% & 90.5\%  & 98.4\%  & 96.0\%  & 92.4\%  & 98.1\%  & 96.3\%  & 93.9\%  \\
\rowcolor[rgb]{0.88,0.97,0.97}
SCOPE (NeurIPS25)  & 97.8\% & 96.0\% & 93.5\%  & 97.7\% & 96.4\% & 93.3\%  & 99.8\%  & 97.8\%  & 94.4\%  & 99.4\%  & 98.2\%  & 95.8\%  \\
\rowcolor[rgb]{0.93,0.88,1.0}
\textbf{TOPS (Ours)} & \textbf{98.3\%} & \textbf{97.1\%} & \textbf{94.6\%} & \textbf{98.9\%} & \textbf{97.3\%} & \textbf{94.7\%} & \textbf{100.0\%} & \textbf{99.1\%} & \textbf{96.4\%} & \textbf{100.6\%} & \textbf{99.1\%} & \textbf{96.6\%} \\
\bottomrule
\end{tabular}
}

    \vspace{-3mm}
\end{table*}

\vspace{-1mm}
\subsection{Principle Instantiation and Analysis}
\vspace{-1mm}
\label{sec:principle_analysis}
We empirically evaluate three principles using representative methods: FastV~\cite{chen2024image} for task relevance, DivPrune~\cite{alvar2025divprune} for semantic diversity, and SCOPE~\cite{deng2025scope} for information coverage. We measure the inference loss increase over the unpruned model to quantify pruning degradation. As shown in Figure~\ref{fig:loss_diff_comparison_mme}, relevance-based pruning performs better at low pruning ratios, while diversity- and coverage-based methods become more effective under aggressive pruning. This suggests that different principles dominate under different budgets, and combining them in TOPS yields the best overall performance. Additional results are reported in Table~\ref{tab:ablation_rdc}.

\section{Method}
\label{sec:method}
\vspace{-2mm}

\subsection{Token Optimal Preservation Set}
\label{sec:method_tops}
\vspace{-2mm}

Utilizing the first-principles derived in Section~\ref{sec:motivation_decomposition}, we design the TOPS construction procedure. At each pruning point, we dynamically select a set of text raters $T_r$ and precompute a pairwise similarity matrix $\mathbf{F}$ over the visual tokens. 
For each candidate token $i$, given the current selected subset $S$, we denote the remaining visual tokens as $U=V\setminus S$ and compute its task relevance $r_i = \frac{1}{|T_r|}\sum_{t\in T_r}\mathrm{Attn}(t\!\rightarrow\!i)$ via text-rater attention, and its information coverage $c_i(U)=\sum_{j\in U}\max\!\left(0,\mathrm{sim}(h_i,h_j)-\max_{k\in S}\mathrm{sim}(h_j,h_k)\right)$. After min-max normalization, denoted by $\tilde{r}_i$ and $\tilde{c}_i(U)$, these two scores are combined into an information preservation score:
\begingroup
\setlength{\abovedisplayskip}{2pt}
\setlength{\belowdisplayskip}{2pt}
\setlength{\abovedisplayshortskip}{0pt}
\setlength{\belowdisplayshortskip}{2pt}
\begin{equation}
\widetilde{\mathrm{info}}_i
= \tilde{r}_i + \lambda\,\tilde{c}_i(U).
\label{eq:info_score}
\end{equation}
\endgroup
Following the diversity principle derived before, we further incorporate a semantic diversity score $d_i(S)=1-\max_{j\in S}\mathrm{sim}(h_i,h_j)$ to select. The subset is then expanded greedily as:
\begingroup
\setlength{\abovedisplayskip}{2pt}
\setlength{\belowdisplayskip}{2pt}
\setlength{\abovedisplayshortskip}{0pt}
\setlength{\belowdisplayshortskip}{2pt}
\begin{equation}
\begin{aligned}
S_{t+1}
&= S_t \cup \left\{
\arg\max_{i\in V\setminus S_t}
\left(
\widetilde{\mathrm{info}}_i
+\alpha\,\tilde{d}_i(S_t)
\right)
\right\}.
\end{aligned}
\label{eq:method_tops_score}
\end{equation}
\endgroup
where $\alpha,\lambda$ are balance factors. We expand the subset until it reaches the target size $K$. The construction is initialized as $S_1=\{\arg\max_i r_i\}$. Since coverage and diversity share the same $\max_{k\in S}\mathrm{sim}(\cdot,h_k)$ term, it is incrementally maintained, introducing negligible overhead. The complete algorithm is in Algorithm~\ref{alg:stage2} in Appendix.

\subsection{Multi-Stage TOPS Implementation}
\label{sec:method_progressive}

TOPS can be applied at multiple pruning points during MLLM inference. To fully exploit its flexibility, we implement a multi-stage pipeline for fine-grained pruning, as illustrated in Figure.~\ref{fig:overview}.

\textbf{Stage~I} applies TOPS after multimodal projector, before visual tokens enter LLM. Let $V^{(0)}=g(f_v(X_v))$ denote the projected visual token set. Since text-rater attention is unavailable at this stage, we use CLS attention to replace it:
\begingroup
\setlength{\abovedisplayskip}{2pt}
\setlength{\belowdisplayskip}{2pt}
\setlength{\abovedisplayshortskip}{0pt}
\setlength{\belowdisplayshortskip}{2pt}
\begin{equation}
\begin{aligned}
V^{(1)}
&= \mathrm{TOPS}\!\left(V^{(0)},\; r^{\mathrm{cls}}\right).
\end{aligned}
\label{eq:stage1}
\end{equation}
\endgroup
where $r^{\mathrm{cls}}$ denotes CLS-based relevance scores. This coarse reduction removes clearly redundant tokens before LLM processing.

\textbf{Stage~II} applies TOPS inside the LLM at a set of designated layers $\mathcal{P}=\{p_1,\ldots,p_L\}$. At each layer $p_l \in \mathcal{P}$, for $l=1,\ldots,L$, we perform:
\begingroup
\setlength{\abovedisplayskip}{2pt}
\setlength{\belowdisplayskip}{2pt}
\setlength{\abovedisplayshortskip}{0pt}
\setlength{\belowdisplayshortskip}{2pt}
\begin{equation}
\begin{aligned}
V^{(l+1)}
&= \mathrm{TOPS}\!\left(V^{(l)},\; r^{(p_l)}\right).
\end{aligned}
\label{eq:stage2}
\end{equation}
\endgroup
where $r^{(p_l)}$ denotes text-rater relevance at layer $p_l$, enabling TOPS to leverage LLM's text-to-visual attention at deeper layers for useful token selection.

\section{Experiments}

\begin{table*}[t]
    \centering
    \centering
\setlength{\tabcolsep}{2pt}
\caption{Performance comparison of different pruning methods on advanced VLM architectures across 8 benchmarks. \textbf{Acc.} denotes the average percentage of baseline performance maintained. \colorbox[rgb]{1.0,0.88,0.88}{Red}: attention-based. \colorbox[rgb]{0.88,0.96,0.88}{Green}: attention\&diversity. \colorbox[rgb]{0.88,0.93,1.0}{Blue}: diversity-based. \colorbox[rgb]{0.93,0.88,1.0}{Purple}: ours.}
\label{tab:pruning_comparison_advanced}
\vspace{-2mm}
\resizebox{\textwidth}{!}{%
\begin{tabular}{l|cccccccc|cc}
\toprule
\textbf{Method} & \textbf{AI2D} & \textbf{POPE} & \textbf{Hall} & \textbf{MME} & \textbf{MMB}$^{\text{EN}}$ & \textbf{MMB}$^{\text{CN}}$ & \textbf{MMStar} & \textbf{SQA} & \textbf{Acc.} & \textbf{Rel.} \\
\midrule
\multicolumn{11}{c}{\textit{Qwen2.5-VL-7B-Instruct — Upper Bound, All 1296 Tokens (100\%)}} \\
\midrule
Baseline & 84.9 & 87.7 & 55.9 & 2301.8 & 84.8 & 82.9 & 65.5 & 86.8 & 73.1 & 100.0\% \\
\midrule
\multicolumn{11}{c}{\textit{Retain 256 Tokens ($\downarrow$ 80.2\%)}} \\
\midrule
\rowcolor[rgb]{1.0,0.88,0.88}
FastV (ECCV24) & 78.4 & 83.0 & \underline{49.1} & 2169.3 & 80.5 & 78.8 & 55.5 & 83.6 & 67.8 & 92.7\% \\
\rowcolor[rgb]{0.88,0.96,0.88}
CDPruner (NeurIPS25) & \textbf{82.2} & 83.5 & 45.3 & \underline{2231.9} & \underline{81.5} & 80.1 & 57.7 & 84.1 & \underline{68.9} & \underline{94.3\%} \\
\rowcolor[rgb]{0.88,0.93,1.0}
DivPrune (CVPR25) & 81.2 & \underline{85.3} & 46.6 & 2167.3 & \textbf{81.8} & \textbf{80.9} & \underline{57.9} & \underline{84.8} & 68.8 & 94.1\% \\
\rowcolor[rgb]{0.93,0.88,1.0}
\textbf{TOPS (Ours)} & \underline{81.5} & \textbf{86.1} & \textbf{50.3} & \textbf{2284.6} & 81.4 & \underline{80.5} & \textbf{59.5} & \textbf{87.2} & \textbf{70.4} & \textbf{96.3\%} \\
\midrule
\multicolumn{11}{c}{\textit{Retain 128 Tokens ($\downarrow$ 90.1\%)}} \\
\midrule
\rowcolor[rgb]{1.0,0.88,0.88}
FastV (ECCV24) & 69.9 & 67.7 & 41.0 & 1596.8 & 66.4 & 68.8 & 43.7 & 79.6 & 57.0 & 78.0\% \\
\rowcolor[rgb]{0.88,0.96,0.88}
CDPruner (NeurIPS25) & \textbf{79.5} & \underline{80.5} & 41.3 & 2033.2 & 78.6 & \underline{78.7} & \underline{53.8} & 82.5 & \underline{65.5} & \underline{89.6\%} \\
\rowcolor[rgb]{0.88,0.93,1.0}
DivPrune (CVPR25) & 75.9 & \textbf{83.9} & \underline{44.5} & \underline{2044.3} & \underline{79.2} & 78.6 & 52.3 & \underline{82.8} & \underline{65.5} & \underline{89.6\%} \\
\rowcolor[rgb]{0.93,0.88,1.0}
\textbf{TOPS (Ours)} & \underline{78.9} & \textbf{83.9} & \textbf{46.8} & \textbf{2217.3} & \textbf{80.2} & \textbf{78.8} & \textbf{53.9} & \textbf{85.3} & \textbf{67.9} & \textbf{92.9\%} \\
\midrule
\multicolumn{11}{c}{\textit{InternVL3-8B — Upper Bound, All 1280 Tokens (100\%)}} \\
\midrule
Baseline & 85.1 & 90.4 & 49.4 & 2369.1 & 85.7 & 85.1 & 68.3 & 97.9 & 74.9 & 100.0\% \\
\midrule
\multicolumn{11}{c}{\textit{Retain 256 Tokens ($\downarrow$ 80.0\%)}} \\
\midrule
\rowcolor[rgb]{1.0,0.88,0.88}
FastV (ECCV24) & \underline{80.5} & 88.7 & \underline{44.0} & \underline{2289.4} & \underline{83.6} & \underline{83.7} & \underline{61.2} & \underline{93.3} & \underline{71.2} & \underline{95.1\%} \\
\rowcolor[rgb]{0.88,0.96,0.88}
CDPruner (NeurIPS25) & 78.8 & 89.1 & 41.5 & 2130.0 & 79.2 & 78.3 & 55.7 & 90.2 & 67.4 & 90.0\% \\
\rowcolor[rgb]{0.88,0.96,0.88}
VisionZip (CVPR25) & 76.0 & 85.6 & 41.0 & 2148.8 & 82.0 & 81.2 & 55.1 & 90.7 & 67.8 & 90.5\% \\
\rowcolor[rgb]{0.88,0.93,1.0}
DivPrune (CVPR25) & 80.3 & \textbf{89.4} & 43.0 & 2178.5 & 81.9 & 80.5 & 58.9 & 91.8 & 69.3 & 92.5\% \\
\rowcolor[rgb]{0.93,0.88,1.0}
\textbf{TOPS (Ours)} & \textbf{81.3} & \underline{89.3} & \textbf{45.8} & \textbf{2302.0} & \textbf{84.8} & \textbf{84.6} & \textbf{62.8} & \textbf{95.5} & \textbf{72.4} & \textbf{96.7\%} \\
\midrule
\multicolumn{11}{c}{\textit{Retain 128 Tokens ($\downarrow$ 90.0\%)}} \\
\midrule
\rowcolor[rgb]{1.0,0.88,0.88}
FastV (ECCV24) & 68.4 & 73.8 & \underline{38.9} & 1806.9 & 73.7 & 73.5 & 47.3 & 83.8 & 60.4 & 80.6\% \\
\rowcolor[rgb]{0.88,0.96,0.88}
CDPruner (NeurIPS25) & 73.1 & 86.9 & 37.4 & 1952.6 & 75.3 & 73.5 & 51.5 & 85.6 & 62.8 & 83.8\% \\
\rowcolor[rgb]{0.88,0.96,0.88}
VisionZip (CVPR25) & 68.5 & 77.9 & 33.6 & 1864.7 & 74.7 & 73.5 & 49.2 & 81.7 & 60.3 & 80.5\% \\
\rowcolor[rgb]{0.88,0.93,1.0}
DivPrune (CVPR25) & \underline{74.2} & \textbf{88.0} & 37.6 & \underline{2051.0} & \underline{78.3} & \underline{75.7} & \underline{52.0} & \underline{87.5} & \underline{64.6} & \underline{86.2\%} \\
\rowcolor[rgb]{0.93,0.88,1.0}
\textbf{TOPS (Ours)} & \textbf{76.1} & \underline{87.5} & \textbf{41.4} & \textbf{2252.0} & \textbf{82.0} & \textbf{80.5} & \textbf{57.1} & \textbf{91.8} & \textbf{68.8} & \textbf{91.9\%} \\
\bottomrule
\end{tabular}
}

    \vspace{-2mm}
\end{table*}

\subsection{Experimental Setup}
\paragraph{Model Architectures.}
We validate TOPS across multiple MLLM architectures, including LLaVA-1.5~\cite{liu2024improved} for image understanding, LLaVA-NeXT~\cite{liu2024llavanext} for high-resolution inputs, and LLaVA-Video~\cite{zhang2024llava} for video tasks. We also evaluate on advanced models Qwen2.5-VL-7B-Instruct~\cite{Qwen2.5-VL} and InternVL3-8B~\cite{zhu2025internvl3}. More experiments are provided in Appendix~\ref{sec:appendix_13b}.

\paragraph{Evaluation Benchmarks.}
We conduct experiments across diverse multimodal benchmarks. For image-based evaluation, we select 8 general VQA benchmarks: GQA~\cite{hudson2019gqa}, ScienceQA-IMG~\cite{lu2022learn}, TextVQA~\cite{singh2019towards}, POPE~\cite{li2023evaluating}, MME~\cite{fu2023mme}, MMBench-EN, MMBench-CN~\cite{liu2024mmbench}, and MM-Vet~\cite{yu2023mm}. Additionally, we evaluate on MMStar~\cite{chen2024we}, AI2D~\cite{kembhavi2016diagram}, and HallusionBench~\cite{guan2024hallusionbench}. For video understanding, we benchmark on MLVU~\cite{zhou2025mlvu}, LongVideoBench~\cite{wu2024longvideobench}, and Video-MME~\cite{fu2025video}.

\begin{table*}[t]
    \vspace{-2mm}
    \centering
    \centering
\setlength{\tabcolsep}{2pt}
\caption{Performance comparison of different methods on LLaVA-Video-7B with 64 frames per video. \textbf{Acc.} denotes average accuracy across 8 metrics of 3 benchmarks. \colorbox[rgb]{1.0,0.88,0.88}{Red}: attention-based. \colorbox[rgb]{0.88,0.93,1.0}{Blue}: diversity-based. \colorbox[rgb]{0.93,0.88,1.0}{Purple}: ours.}
\label{tab:video_comparison}
\vspace{-2mm}
\resizebox{\textwidth}{!}{%
\begin{tabular}{l|c|ccc|cccc|cc}
\toprule
\textbf{Method} & \textbf{MLVU} & \multicolumn{3}{c|}{\textbf{LongVideoBench}} & \multicolumn{4}{c|}{\textbf{Video-MME}} & \textbf{Acc.} & \textbf{Rel.} \\
\textbf{Metric} & \textbf{m-avg} & \textbf{val} & \textbf{perception} & \textbf{relation} & \textbf{w/o sub} & \textbf{short} & \textbf{medium} & \textbf{long} &  &  \\
\midrule
\multicolumn{11}{c}{\textit{Upper Bound, All $64 \times 169$ Tokens (100\%)}} \\
\midrule
Baseline & 67.7 & 59.0 & 65.0 & 53.8 & 63.6 & 76.6 & 61.2 & 53.1 & 62.5 & 100.0\% \\
\midrule
\multicolumn{11}{c}{\textit{Retain $64 \times 64$ Tokens ($\downarrow$ 62.1\%)}} \\
\midrule
\rowcolor[rgb]{1.0,0.88,0.88}
FastV (ECCV24) & 63.9 & 56.1 & 60.6 & 52.1 & \underline{61.9} & \underline{73.6} & 59.3 & \textbf{52.7} & 60.0 & 96.0\% \\
\rowcolor[rgb]{1.0,0.88,0.88}
SparseVLM (ICML25) & \underline{65.5} & 56.0 & 61.0 & 51.7 & 61.0 & 73.0 & 58.8 & 51.2 & 59.8 & 95.7\% \\
\rowcolor[rgb]{0.88,0.93,1.0}
DART (EMNLP25) & 64.1 & 57.5 & 62.1 & \underline{53.5} & 61.6 & 73.0 & \underline{59.9} & 51.9 & 60.5 & 96.8\% \\
\rowcolor[rgb]{0.88,0.93,1.0}
DivPrune (CVPR25) & 64.1 & \textbf{58.6} & \underline{64.2} & \textbf{53.7} & 61.1 & 72.9 & 59.3 & 51.2 & \underline{60.6} & \underline{97.0\%} \\
\rowcolor[rgb]{0.93,0.88,1.0}
\textbf{TOPS (Ours)} & \textbf{66.4} & \underline{57.6} & \textbf{64.6} & 51.4 & \textbf{63.0} & \textbf{75.3} & \textbf{61.4} & \underline{52.3} & \textbf{61.5} & \textbf{98.4\%} \\
\midrule
\multicolumn{11}{c}{\textit{Retain $64 \times 32$ Tokens ($\downarrow$ 81.1\%)}} \\
\midrule
\rowcolor[rgb]{1.0,0.88,0.88}
FastV (ECCV24) & 58.5 & 52.4 & 57.0 & 48.5 & 56.0 & 63.8 & 55.9 & 48.4 & 55.1 & 88.2\% \\
\rowcolor[rgb]{1.0,0.88,0.88}
SparseVLM (ICML25) & 60.7 & 53.7 & 58.1 & 49.9 & 59.0 & 69.8 & 56.9 & \underline{50.3} & 57.3 & 91.7\% \\
\rowcolor[rgb]{0.88,0.93,1.0}
DART (EMNLP25) & 61.1 & 54.1 & 57.8 & 50.8 & 58.1 & 67.3 & 57.1 & 50.0 & 57.0 & 91.2\% \\
\rowcolor[rgb]{0.88,0.93,1.0}
DivPrune (CVPR25) & \underline{61.5} & \underline{56.4} & \underline{62.1} & \textbf{51.4} & \underline{59.3} & \underline{69.9} & \underline{57.9} & 50.2 & \underline{58.6} & \underline{93.8\%} \\
\rowcolor[rgb]{0.93,0.88,1.0}
\textbf{TOPS (Ours)} & \textbf{64.3} & \textbf{56.9} & \textbf{63.5} & \underline{51.1} & \textbf{61.2} & \textbf{72.3} & \textbf{58.8} & \textbf{52.6} & \textbf{60.1} & \textbf{96.2\%} \\
\midrule
\multicolumn{11}{c}{\textit{Retain $64 \times 16$ Tokens ($\downarrow$ 90.5\%)}} \\
\midrule
\rowcolor[rgb]{1.0,0.88,0.88}
FastV (ECCV24) & 52.8 & 46.6 & 48.8 & 44.7 & 50.0 & 55.0 & 50.0 & 45.0 & 49.1 & 78.6\% \\
\rowcolor[rgb]{1.0,0.88,0.88}
SparseVLM (ICML25) & 52.0 & 47.6 & 53.0 & 42.8 & 49.8 & 53.8 & 49.3 & 46.3 & 49.3 & 78.9\% \\
\rowcolor[rgb]{0.88,0.93,1.0}
DART (EMNLP25) & 56.7 & 51.8 & 56.8 & \underline{47.5} & 55.3 & 64.8 & 52.9 & 48.1 & 54.2 & 86.7\% \\
\rowcolor[rgb]{0.88,0.93,1.0}
DivPrune (CVPR25) & \underline{58.6} & \underline{52.1} & \underline{57.6} & 47.2 & \underline{56.7} & \underline{67.7} & \underline{54.2} & \underline{48.2} & \underline{55.3} & \underline{88.5\%} \\
\rowcolor[rgb]{0.93,0.88,1.0}
\textbf{TOPS (Ours)} & \textbf{60.7} & \textbf{54.3} & \textbf{60.8} & \textbf{48.6} & \textbf{59.4} & \textbf{70.9} & \textbf{57.4} & \textbf{49.8} & \textbf{57.7} & \textbf{92.3\%} \\
\bottomrule
\end{tabular}
}

%
%
   
\end{table*}

\paragraph{Comparison Methods.}
We compare TOPS with recent methods, including FastV~\cite{chen2024image}, PyramidDrop~\cite{xing2024pyramiddrop}, SparseVLM~\cite{zhang2024sparsevlm}, DivPrune~\cite{alvar2025divprune}, DART~\cite{wen2025stop}, VisionZip~\cite{yang2025visionzip}, TRIM~\cite{song2025less}, PruMerge+~\cite{shang2025llava}, SCOPE~\cite{deng2025scope}, and CDPruner~\cite{zhang2025cdpruner}.

\subsection{TOPS for LLaVA and LLaVA-NeXT}

We first evaluate TOPS on LLaVA-1.5 and LLaVA-NeXT, widely adopted for benchmarking token pruning (Table~\ref{tab:llava_series_summary}; full per-benchmark results in Appendix~\ref{sec:appendix_13b}.). On LLaVA-1.5-7B, TOPS retains 98.3\% of the original performance at 77.8\% compression, surpassing SCOPE by 0.5\%. At 64 tokens, attention-based methods degrade by over 25\%, while TOPS only decreases by 1.2\%. Even at 32 tokens (5.6\% retained), TOPS maintains 94.6\%, outperforming SCOPE by 1.1\%. On LLaVA-NeXT with 2,880 visual tokens, TOPS achieves 100.0\% performance at 77.8\% compression, and retains 99.1\% and 96.4\% at 88.9\% and 94.4\% reduction, outperforming SCOPE by 1.3\% and 2.0\%.

\vspace{-1pt}
\subsection{TOPS for Qwen2.5-VL and InternVL3}

To verify generalizability, we further evaluate TOPS on Qwen2.5-VL-7B-Instruct and InternVL3-8B (Table~\ref{tab:pruning_comparison_advanced}), two architectures with different visual encoders and fusion strategies.At $\sim$80\% pruning, TOPS preserves 96.3\% and 96.7\% of the original performance, surpassing the best competing methods by 2.0\% and 1.6\%, respectively. At $\sim$90\% pruning, TOPS retains 92.9\% and 91.9\% accuracy while its advantage amplifies, reaching +3.3\% over CDPruner/DivPrune on Qwen2.5-VL and +5.7\% over DivPrune on InternVL3. Notably, the best baseline differs across architectures (CDPruner on Qwen2.5-VL vs.\ FastV on InternVL3), yet TOPS consistently ranks first, demonstrating architecture-agnostic effectiveness. TOPS achieves the best scores across all settings On HallusionBench, indicating stronger resistance to hallucination.

\vspace{-1pt}
\subsection{TOPS for LLaVA-Video}

Video understanding is highly redundant because multi-frame inputs introduce many visual tokens. We apply TOPS to LLaVA-Video with up to 64 frames at $384{\times}384$ resolution, producing over 10K visual tokens (Table~\ref{tab:video_comparison}). TOPS remains robust under aggressive compression, retaining 98.4\% and 96.2\% performance at 62.1\% and 81.1\% token reduction, respectively. Even with only 16 tokens per frame, TOPS maintains 92.3\%, while FastV drops to 78.6\%.

\begin{table}
    \centering
    \vspace{2mm}
    \small
\setlength{\tabcolsep}{2pt}

\caption{Efficiency analysis on LLaVA-NeXT-7B (POPE benchmark). Latency in ms; Memory in GB.}
\label{tab:efficiency}

\vspace{-8pt}

\resizebox{\linewidth}{!}{%
\begin{tabular}{l|c|c|c|c|c}
\toprule
\textbf{Method} & \textbf{\#Tok} & \textbf{FLOPs (T)} & \textbf{Lat. (ms)} & \textbf{Mem. (GB)} & \textbf{F1} \\
\midrule
Baseline & 2880 & 41.7 & 265 & 16.7 & 86.8 \\
\midrule
\rowcolor[rgb]{1.0,0.88,0.88} FastV (ECCV24) & 320 & \underline{4.4 ($\times$9.5)} & 77 & \underline{15.6} & 49.5 \\
\rowcolor[rgb]{1.0,0.88,0.88} PDrop (CVPR25) & 320 & \underline{4.4 ($\times$9.4)} & 67 & \underline{15.6} & 60.8 \\
\rowcolor[rgb]{1.0,0.88,0.88} SparseVLM (ICML25) & 320 & \underline{4.4 ($\times$9.5)} & 101 & 18.6 & \underline{85.3} \\
\rowcolor[rgb]{0.88,0.96,0.88} PruMerge+ (ICCV25) & 320 & \textbf{4.2 ($\times$9.9)} & \textbf{54} & \textbf{14.8} & 79.5 \\
\rowcolor[rgb]{0.88,0.96,0.88} VisionZip (CVPR25) & 320 & \textbf{4.2 ($\times$9.9)} & \underline{60} & \textbf{14.8} & 82.3 \\
\rowcolor[rgb]{0.93,0.88,1.0} \textbf{TOPS (Ours)} & 320 & \textbf{4.2 ($\times$9.9)} & 85 & \textbf{14.8} & \textbf{86.3} \\
\bottomrule
\end{tabular}
}
    \vspace{-4mm}
\end{table}

\begin{table*}[t]
    \centering
    \small
\setlength{\tabcolsep}{3pt}
\caption{Ablation of token selection criteria.}
\label{tab:ablation_rdc}
\vspace{-7pt}
\resizebox{\linewidth}{!}{%
\begin{tabular}{lcccccccc|cc}
\toprule
\textbf{Criteria} & \textbf{GQA} & \textbf{SQA} & \textbf{TVQA} & \textbf{POPE} & \textbf{MME} & \textbf{MMB}$^{\text{EN}}$ & \textbf{MMB}$^{\text{CN}}$ & \textbf{MMVet} & \textbf{Acc.} & \textbf{Rel.} \\
\midrule
\rowcolor[rgb]{0.90,0.96,0.90}
Relevance only         & 53.5 & \textbf{69.2} & 53.9 & 77.5          & 1347.4             & \underline{60.6} & 54.9             & 24.5             & 57.5             & 91.1\% \\
\rowcolor[rgb]{0.90,0.96,0.90}
Diversity only         & 55.0 & 67.5          & 53.1 & \textbf{84.7} & 1355.4             & 58.1             & 52.4             & 27.5             & 58.3             & 92.4\% \\
\rowcolor[rgb]{0.90,0.96,0.90}
Coverage only          & \underline{56.6} & 68.9 & 51.4 & 83.2        & 1358.1             & 59.5             & 51.6             & 24.8             & 58.0             & 91.9\% \\
\midrule
\rowcolor[rgb]{0.83,0.90,0.96}
R + D                  & 55.6 & 69.0          & 54.4          & 81.7             & 1365.0             & 60.2             & \underline{55.2} & 27.8             & 59.0             & 93.5\% \\
\rowcolor[rgb]{0.83,0.90,0.96}
R + C                  & 56.4 & \underline{69.1} & \underline{54.7} & 82.1        & \underline{1371.6} & \textbf{61.0}    & \textbf{55.6}    & \underline{28.0} & \underline{59.4} & \underline{94.1\%} \\
\rowcolor[rgb]{0.83,0.90,0.96}
D + C                  & 55.7 & 68.4          & 53.0          & \underline{84.4} & 1368.6             & 58.7             & 52.4             & 26.2             & 58.4             & 92.6\% \\
\midrule
\rowcolor[rgb]{0.93,0.88,1.0}
\textbf{TOPS}          & \textbf{56.7} & 68.8 & \textbf{54.9} & 83.5 & \textbf{1384.7} & 59.5 & 55.1 & \textbf{29.7} & \textbf{59.7} & \textbf{94.6\%} \\
\bottomrule
\end{tabular}
}

    \vspace{-4mm}
\end{table*}

\begin{figure*}[t]
    \centering
    \includegraphics[width=\linewidth]{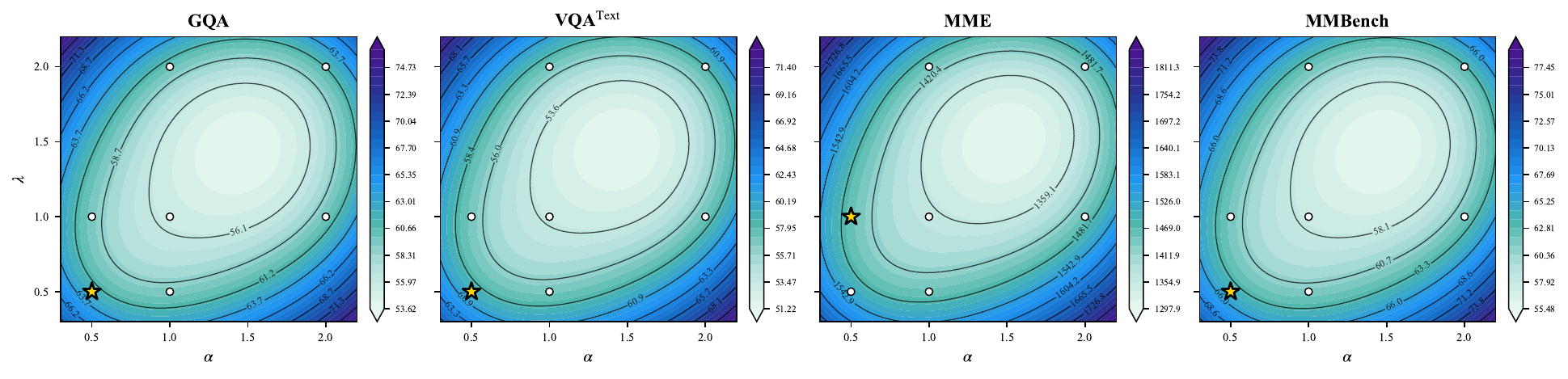}
    \vspace{-8mm}
    \caption{\textbf{Hyperparameter sensitivity of $\alpha$ and $\lambda$.} Contour plots across seven ($\alpha$, $\lambda$) configurations at 64 tokens on LLaVA-1.5-7B. \textbf{Star}: optimal; \textbf{white dots}: other configurations. The optimal ($\alpha$, $\lambda$) generally falls within $[0.5, 1]$.}
    \label{fig:ablation_hyperparams}
\end{figure*}

\subsection{Computational Efficiency}
\label{sec:computational_efficiency}

To demonstrate the efficiency of TOPS, we conduct a comparative analysis against other methods in terms of FLOPs, CUDA latency, GPU memory and F1 score on LLaVA-NeXT-7B. 
Experiments are performed on a single NVIDIA A800-80GB GPU. We use POPE for evaluating inference efficiency, as it contains questions of similar length and involves only one prefill and one decode stage.

As shown in Table~\ref{tab:efficiency}, when the number of visual tokens is reduced from 2,880 to 320, TOPS achieves nearly a ${\times}10$ reduction in FLOPs. In terms of runtime latency, TOPS reduces prefill time and decode time by ${\times}3.12$ and ${\times}1.05$, significantly improving real-world inference efficiency. In addition to latency, TOPS reduces GPU memory usage by 1.9GB. Compared to other methods, TOPS achieves the best performance (86.3 vs.\ 85.3) while maintaining comparable or even better efficiency.

\subsection{Ablation Studies}
\label{sec:ablation_studies}

We conduct a series of ablation studies to analyze the key design choices of TOPS on LLaVA-1.5-7B. 
Table~\ref{tab:ablation_stages} examines contribution of each stage in the two-stage pipeline, showing that TOPS's full two-stage design yields the best results. We further evaluate all combinations of Relevance (R), Diversity (D), and Coverage (C) at 32 tokens (Table~\ref{tab:ablation_rdc}), where TOPS achieves the best performance. 

\begin{table}[h]
\centering
\vspace{-2mm}
\small
\setlength{\tabcolsep}{3pt}

\caption{Ablation of two stages.}
\vspace{-3mm}
\label{tab:ablation_stages}
\vspace{-1pt}

\resizebox{\linewidth}{!}{%
\begin{tabular}{lccccc}
\toprule
\textbf{Setting} & \textbf{GQA} & \textbf{POPE} & \textbf{MME} & \textbf{MMB}$^{\text{EN}}$ & \textbf{MMB}$^{\text{CN}}$ \\
\midrule
\rowcolor[rgb]{0.90,0.96,0.90}
S1-only & \underline{59.2} & 86.2 & \underline{1444.0} & 61.1 & \underline{56.5} \\
\rowcolor[rgb]{0.83,0.90,0.96}
S2-only & 59.1 & \underline{86.5} & 1412.6 & \underline{61.7} & 56.0 \\
\midrule
\rowcolor[rgb]{0.93,0.88,1.0}
\textbf{TOPS} & \textbf{60.5} & \textbf{86.8} & \textbf{1482.7} & \textbf{62.5} & \textbf{57.2} \\
\bottomrule
\end{tabular}%
}

\vspace{-2mm}
\end{table}

As shown in Figure~\ref{fig:ablation_token_budget}, TOPS consistently outperforms FastV, DivPrune, and SCOPE across all five token budgets, and its advantage widens under more aggressive compression. Figure~\ref{fig:ablation_hyperparams} visualizes the sensitivity of $\alpha$ (diversity weight) and $\lambda$ (coverage weight) across 8 benchmarks at 64 tokens. Across seven ($\alpha$, $\lambda$) configurations, the optimal values consistently fall within $[0.5, 1]$ for both parameters, with the exception of POPE, where a larger $\alpha{=}2$ is preferred due to its binary question format that benefits from stronger diversity. The overall performance variation remains within 1--2\% across all configurations, confirming that TOPS is robust to hyperparameter choices. Additional ablations on pruning layer positions are provided in Appendix~\ref{sec:appendix_pruning_layers}.

\begin{figure}[h]
    \centering
    \begin{subfigure}{0.48\linewidth}
        \centering
        \includegraphics[width=\linewidth]{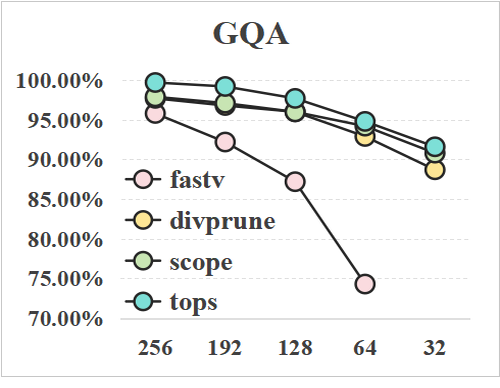}
    \end{subfigure}
    \hfill
    \begin{subfigure}{0.48\linewidth}
        \centering
        \includegraphics[width=\linewidth]{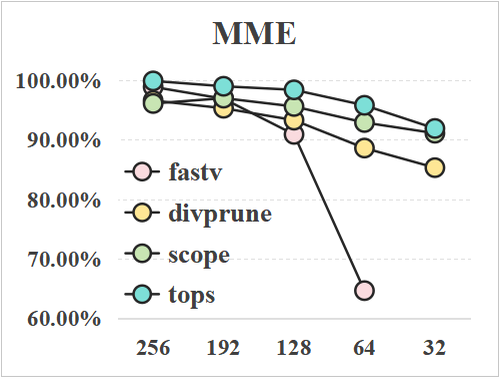}
    \end{subfigure}
    \\
    \begin{subfigure}{0.48\linewidth}
        \centering
        \includegraphics[width=\linewidth]{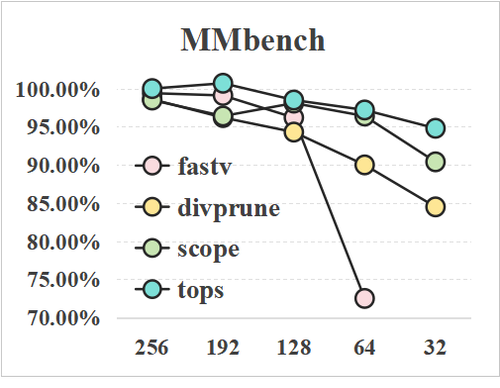}
    \end{subfigure}
    \hfill
    \begin{subfigure}{0.48\linewidth}
        \centering
        \includegraphics[width=\linewidth]{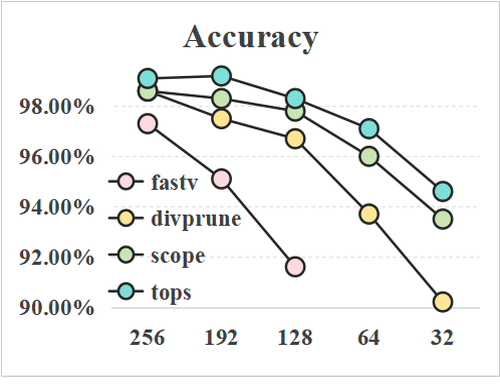}
    \end{subfigure}
    \vspace{-1mm}
    \caption{\textbf{Robustness across token budgets.} Relative performance (\%) of FastV, DivPrune, SCOPE and TOPS at five budgets on LLaVA-1.5-7B.}
    \label{fig:ablation_token_budget}
    \vspace{-5mm}
\end{figure}

\section{Conclusion}
In this work, we revisit visual token pruning from a first-principles perspective and identify its fundamental objective. Based on an information-theoretic analysis, we derive three key principles---task relevance, information coverage, and semantic diversity---and propose TOPS, a plug-and-play pruning method that constructs compact yet informative token subsets. Extensive experiments across multiple MLLMs and benchmarks demonstrate that TOPS consistently achieves superior performance under aggressive token reduction while maintaining strong generalization across model architectures and tasks. These results show that effective token pruning should not rely solely on heuristic importance scores, but should jointly preserve task-relevant, representative, and non-redundant visual information. We believe that our work provides a principled foundation and offers useful guidance for future research on token pruning and efficient multimodal inference.

\section*{Limitations}

While TOPS demonstrates consistent improvements across diverse LVLMs and benchmarks, several limitations remain.

\paragraph{Greedy construction overhead.}
TOPS constructs the preserved token set greedily, updating diversity and coverage scores after each selection step. Although both scores share the same max-similarity structure and can be maintained incrementally, the greedy loop still introduces $O(KN)$ additional operations per pruning layer (where $K$ is the target budget and $N$ is the current token count), which is non-negligible at very aggressive budgets or when pruning is applied at many layers. Future work could explore approximate or parallel variants to further reduce this overhead.

\paragraph{Attention-based relevance signal.}
The task relevance score in TOPS relies on cross-modal attention weights as a proxy for query-conditioned importance. While this signal is readily available in standard transformer architectures, it may be less reliable in models that use alternative attention mechanisms (e.g., linear attention or sparse attention), or in very early layers where text-visual attention has not yet matured. The Stage~I relevance estimation similarly relies on CLS attention, which may not generalize equally well to vision encoders that lack a dedicated CLS token.

\paragraph{Fixed hyperparameters across layers and tasks.}
The balance coefficients $\alpha$ (diversity) and $\lambda$ (coverage) are set globally and kept fixed across all pruning layers and all tasks. In practice, the optimal trade-off between relevance, coverage, and diversity may vary with pruning depth, token budget, and task type. An adaptive scheme that adjusts these weights per layer or per query could further improve performance, particularly under extreme compression ratios.

\paragraph{Evaluation scope.}
Our experiments focus on standard vision-language benchmarks covering image understanding, video understanding, and OCR-heavy tasks. Performance in highly specialized domains (e.g., medical imaging, remote sensing, or dense captioning with hundreds of objects) has not been systematically evaluated, and the generalization of the three-criterion framework to such settings remains an open question.

\bibliography{main}

\clearpage
\appendix
\begin{center}
{\large\textbf{TOPS: First-Principles Visual Token Pruning via Constructing Token Optimal Preservation Sets for Efficient MLLM Inference}}\\[0.4em]
{\large Appendix}
\end{center}

\vspace{2pt}
\noindent Appendix~\ref{sec:appendix_setup} provides comprehensive details about model architectures, evaluation benchmarks, and baseline methods.
Appendix~\ref{sec:appendix_algorithm} presents the complete TOPS algorithm and theoretical foundations.
Appendix~\ref{sec:appendix_implementation} describes implementation details including hyperparameters and pruning schedules.
Appendix~\ref{sec:appendix_13b} reports additional results on LLaVA-1.5-7B, LLaVA-NeXT-7B, LLaVA-1.5-13B, and LLaVA-NeXT-13B.
Appendix~\ref{sec:appendix_ablation} provides additional ablation studies.
Appendix~\ref{sec:appendix_empirical} presents additional empirical analyses supporting the motivation of TOPS.
Appendix~\ref{sec:appendix_visualization} presents qualitative comparisons of token selections.
Appendix~\ref{sec:appendix_radar} provides per-benchmark radar visualizations across all models and compression ratios.

\section{Experimental Setup Details}
\label{sec:appendix_setup}

\subsection{Model Architectures}
\label{sec:appendix_setup_models}

\paragraph{LLaVA-1.5~\cite{liu2024improved}.}
We evaluate the LLaVA-1.5 architecture, which combines a CLIP ViT-L/14 vision encoder with Vicuna-7B/13B language models through a two-layer MLP projector. For our experiments, we process images at $336\times336$ resolution, yielding 576 visual tokens ($24\times24$ spatial grid). We conduct experiments on both 7B and 13B model scales.

\paragraph{LLaVA-NeXT~\cite{liu2024llavanext}.}
This iteration introduces adaptive resolution handling to accommodate higher-quality visual inputs. The model dynamically partitions high-resolution images into multiple tiles and processes each tile through the vision encoder independently. For controlled evaluation, we standardize the input to $672\times672$ resolution, generating 2,880 visual tokens.

\paragraph{LLaVA-Video~\cite{zhang2024llava}.}
We extend our evaluation to this video-specialized variant that processes temporal sequences of frames. The architecture employs SigLIP as the vision backbone and samples 64 frames per video clip at $384\times384$ resolution, producing an initial token count of 10,816 visual tokens.

\paragraph{InternVL3-8B~\cite{zhu2025internvl3}.}
We evaluate InternVL3, adopting a ViT-MLP-LLM design with native multimodal pre-training. For our experiments, we configure the input resolution to $448\times448$, producing 1,280 visual tokens.

\paragraph{Qwen2.5-VL-7B-Instruct~\cite{Qwen2.5-VL}.}
We assess Qwen2.5-VL, featuring a redesigned Vision Transformer with window attention, SwiGLU activations, and RMSNorm, built upon the Qwen2.5 language model. We evaluate the 7B-Instruct variant using its default dynamic resolution settings.

\subsection{Evaluation Benchmarks}
\label{sec:appendix_setup_benchmarks}

\paragraph{GQA~\cite{hudson2019gqa}.} A visual reasoning benchmark grounded in scene graphs. We report accuracy on the balanced test split with 12,578 questions.
\paragraph{ScienceQA-IMG~\cite{lu2022learn}.} A multimodal multiple-choice benchmark covering diverse scientific subjects.
\paragraph{TextVQA~\cite{singh2019towards}.} Evaluates reading and reasoning about text in images. The validation set has 5,000 questions.
\paragraph{POPE~\cite{li2023evaluating}.} Polling-based Object Probing Evaluation assesses object hallucination using binary yes/no questions. We use the adversarial split with 3,000 questions.
\paragraph{MME~\cite{fu2023mme}.} A comprehensive benchmark evaluating both perception and cognition abilities, containing 14 subtasks.
\paragraph{MMBench~\cite{liu2024mmbench}.} A systematically-designed benchmark covering 20 ability dimensions. We report results on both English and Chinese test sets.
\paragraph{MM-Vet~\cite{yu2023mm}.} Focuses on integrated multimodal capabilities across 6 core abilities with 218 curated examples.
\paragraph{MMStar~\cite{chen2024we}.} A comprehensive vision-language benchmark evaluating diverse capabilities including coarse and fine-grained perception.
\paragraph{AI2D~\cite{kembhavi2016diagram}.} A diagram QA benchmark with over 5,000 science diagrams and 15,000 questions.
\paragraph{HallusionBench~\cite{guan2024hallusionbench}.} Evaluates vision-language models' susceptibility to language hallucination and visual illusion.
\paragraph{MLVU~\cite{zhou2025mlvu}.} A multi-task long video understanding benchmark. We report the mean average score (m-avg) across all subtasks.
\paragraph{LongVideoBench~\cite{wu2024longvideobench}.} A long-context video-language benchmark with interleaved video-language inputs of up to one hour.
\paragraph{Video-MME~\cite{fu2025video}.} A comprehensive video multimodal evaluation benchmark. We report results in the no-subtitle setting across short, medium, and long video splits.

\subsection{Baseline Methods}
\label{sec:appendix_setup_baselines}

We compare TOPS against 10 recent training-free visual token pruning methods.

\paragraph{FastV~\cite{chen2024image}.} An in-LLM pruning method that ranks visual tokens by text-to-visual cross-modal attention scores at layer 2 and removes those with the lowest scores.
\paragraph{PyramidDrop~\cite{xing2024pyramiddrop}.} A progressive in-LLM method that drops a fixed fraction of visual tokens at the end of each decoder stage based on attention importance.
\paragraph{SparseVLM~\cite{zhang2024sparsevlm}.} A text-aware in-LLM method that identifies high-quality text rater tokens and uses their cross-modal attention patterns to guide visual token sparsification.
\paragraph{PruMerge+~\cite{shang2025llava}.} A pre-LLM pruning-and-merging method that computes token importance via attention sparsity and merges similar tokens using $k$-nearest-neighbor matching.
\paragraph{TRIM~\cite{song2025less}.} A pre-LLM method that leverages CLIP-based vision-text similarity to score and retain visual tokens with the highest cross-modal relevance.
\paragraph{VisionZip~\cite{yang2025visionzip}.} A pre-LLM method that selects dominant tokens based on CLS-to-patch attention scores and additionally identifies contextual tokens through clustering.
\paragraph{DART~\cite{wen2025stop}.} A diversity-based in-LLM method that iteratively selects the most diverse tokens by choosing candidates with the lowest similarity to already-selected ones.
\paragraph{DivPrune~\cite{alvar2025divprune}.} A diversity-based pre-LLM method that reformulates token selection as a max-min diversity problem (MMDP).
\paragraph{SCOPE~\cite{deng2025scope}.} A hybrid method combining CLS-based saliency scoring with a submodular coverage term that penalizes semantically redundant tokens.
\paragraph{CDPruner~\cite{zhang2025cdpruner}.} A conditional-diversity method that defines similarity between visual tokens conditioned on the user instruction.

\section{Algorithm and Theoretical Analysis}
\label{sec:appendix_algorithm}

\subsection{Complete Algorithm}
\label{sec:appendix_algorithm_complete}

Algorithm~\ref{alg:stage1} provides the vision-side pruning procedure of TOPS, which is applied immediately after the multimodal projector and before the LLM. Algorithm~\ref{alg:stage2} provides the in-LLM pruning procedure at a designated pruning layer, which is applied progressively across selected LLM layers.

\subsection{Theoretical Analysis of Coverage and Diversity}
\label{sec:appendix_submodularity}

We formally establish the theoretical properties of the coverage and diversity criteria used in TOPS.

\begin{proposition}[Submodularity of Coverage]
\label{prop:coverage_submodular}
The coverage function
\[
F(S)=\sum_{v_i\in V}\max_{s\in S}\operatorname{sim}(v_i,s)
\]
is submodular.
\end{proposition}

\begin{proof}
For any $A\subseteq B\subseteq V$ and $x\in V\setminus B$, define
\[
m_A(v_i)=\max_{a\in A}\operatorname{sim}(v_i,a),\qquad
m_B(v_i)=\max_{b\in B}\operatorname{sim}(v_i,b).
\]
Since $A\subseteq B$, we have $m_A(v_i)\le m_B(v_i)$ for every $v_i\in V$.
The marginal gain of adding $x$ to $A$ for token $v_i$ is
\[
\Delta_x(v_i\mid A)=\max\!\bigl(\operatorname{sim}(v_i,x)-m_A(v_i),\;0\bigr),
\]
and likewise $\Delta_x(v_i\mid B)=\max\!\bigl(\operatorname{sim}(v_i,x)-m_B(v_i),\;0\bigr)$.
Since $m_A(v_i)\le m_B(v_i)$, it follows that $\Delta_x(v_i\mid A)\ge\Delta_x(v_i\mid B)$ for every $v_i\in V$.
Summing over all tokens:
\[
F(A\cup\{x\})-F(A)\;\ge\;F(B\cup\{x\})-F(B),
\]
which is the submodularity condition.
\end{proof}

\begin{remark}[Diminishing Marginal Contribution of Diversity]
\label{rem:diversity_diminishing}
The diversity score $d_i(S)=1-\max_{j\in S}\operatorname{sim}(h_i,h_j)$ satisfies the following \emph{diminishing marginal contribution property}: for any $S\subseteq T$ and $i\notin T$,
$d_i(S)\ge d_i(T)$.
Since $S\subseteq T$, $\max_{j\in T}\operatorname{sim}(h_i,h_j)\ge\max_{j\in S}\operatorname{sim}(h_i,h_j)$, hence
$d_i(T)=1-\max_{j\in T}\operatorname{sim}(h_i,h_j)\le d_i(S)$.
This guarantees that as the retained set grows, each new token contributes progressively less diversity, providing the same intuitive justification for greedy construction as submodularity does for coverage.
\end{remark}

\begin{algorithm}[H]
\caption{TOPS Stage~I --- Vision-Side Token Pruning}
\label{alg:stage1}
\scriptsize
\begin{algorithmic}[1]
\REQUIRE
\begin{tabular}[t]{@{}l@{}}
Projected visual tokens $V^{(0)}=\{h_i^{(0)}\}_{i=1}^{N}$, \\
CLS-to-patch attention scores $a_i$, target budget $M_0$, \\
balance factors $\alpha_1,\lambda_1$
\end{tabular}
\ENSURE Coarsely pruned visual token set $V^{(1)}$

\STATE
\begin{tabular}[t]{@{}l@{}}
Precompute similarity matrix $\mathbf{F}^{(0)}$ from \\
projected visual tokens $\{h_i^{(0)}\}_{i=1}^{N}$
\end{tabular}

\STATE
\begin{tabular}[t]{@{}l@{}}
$S \leftarrow \{i_0\}$ where $i_0=\arg\max_i a_i$ \\
\COMMENT{seed: highest vision-side relevance}
\end{tabular}

\STATE
\begin{tabular}[t]{@{}l@{}}
Initialize $\mathrm{maxsim}_j \leftarrow \mathbf{F}^{(0)}[j,i_0]$ \\
for all $j \notin S$
\end{tabular}

\WHILE{$|S| < M_0$}
    \FOR{each token $i \in V^{(0)} \setminus S$}
        \STATE $\mathrm{div}_i^{(0)} \leftarrow 1-\max_{j\in S}\mathbf{F}^{(0)}[i,j]$

        \STATE
        \begin{tabular}[t]{@{}l@{}}
        $\mathrm{cov}_i^{(0)} \leftarrow \displaystyle\sum_{j\in V^{(0)}\setminus S}$ \\
        $\qquad\max\!\big(0,\,\mathbf{F}^{(0)}[i,j]
        -\max_{k\in S}\mathbf{F}^{(0)}[j,k]\big)$
        \end{tabular}
    \ENDFOR

    \STATE
    \begin{tabular}[t]{@{}l@{}}
    Normalize each criterion by its mean: \\
    $\tilde{a}_i,\;\widetilde{\mathrm{div}}_i^{(0)},\;
    \widetilde{\mathrm{cov}}_i^{(0)}$
    \end{tabular}

    \STATE
    \begin{tabular}[t]{@{}l@{}}
    $\mathrm{score}_i^{(0)} \leftarrow
    \tilde{a}_i+\alpha_1\,\widetilde{\mathrm{div}}_i^{(0)}$ \\
    $\qquad+\lambda_1\,\widetilde{\mathrm{cov}}_i^{(0)}$
    \end{tabular}

    \STATE $i^\star \leftarrow \arg\max_{i\in V^{(0)}\setminus S}\mathrm{score}_i^{(0)}$
    \STATE $S \leftarrow S\cup\{i^\star\}$

    \STATE
    \begin{tabular}[t]{@{}l@{}}
    Update $\mathrm{maxsim}_j \leftarrow
    \max(\mathrm{maxsim}_j,\mathbf{F}^{(0)}[j,i^\star])$ \\
    for all $j\notin S$
    \end{tabular}
\ENDWHILE

\STATE $V^{(1)} \leftarrow S$
\STATE \textbf{return} $V^{(1)}$
\end{algorithmic}
\end{algorithm}

\section{Implementation Details}
\label{sec:appendix_implementation}

\subsection{Codebase}
\label{sec:appendix_implementation_codebase}

We implement TOPS on top of the official LLaVA codebase\footnote{\url{https://github.com/haotian-liu/LLaVA}} for image-based LLaVA models (LLaVA-1.5-7B and 13B). For LLaVA-NeXT and its high-resolution variants, we build on the LLaVA-NeXT codebase\footnote{\url{https://github.com/LLaVA-VL/LLaVA-NeXT}}. For LLaVA-Video, we adopt the same LLaVA-NeXT codebase~\cite{zhang2024llava} and use \texttt{lmms-eval}\footnote{\url{https://github.com/EvolvingLMMs-Lab/lmms-eval}} for video benchmark evaluation. For advanced architectures (Qwen2.5-VL and InternVL3), we integrate TOPS via \texttt{VLMEvalKit}\footnote{\url{https://github.com/open-compass/VLMEvalKit}} to enable unified evaluation across all benchmarks.

\subsection{Hyperparameters}
\label{sec:appendix_implementation_hyperparams}

Unless otherwise specified, we set the balance factors $\alpha$ (diversity weight) and $\lambda$ (coverage weight) through validation on a small held-out subset of POPE, MME, and GQA. Table~\ref{tab:hyperparams_model} reports the $(\alpha, \lambda)$ pairs used for LLaVA-1.5 and LLaVA-NeXT at each pruning ratio. Across all settings, $\alpha = 0.5$ is kept fixed, while $\lambda$ varies slightly between $0.4$ and $1.0$ depending on the compression level. For Stage~I we uniformly use $\alpha_1 = 0.5$, $\lambda_1 = 0.5$ across all models.

\begin{table*}[h]
\centering
\caption{Hyperparameter settings ($\alpha$, $\lambda$) used for each model and pruning ratio. $\alpha$ is the diversity weight and $\lambda$ is the coverage weight.}
\label{tab:hyperparams_model}
\setlength{\tabcolsep}{8pt}
\begin{tabular}{lcccc}
\toprule
\textbf{Pruning Ratio} &
  \textbf{LLaVA-1.5-7B} &
  \textbf{LLaVA-1.5-13B} &
  \textbf{LLaVA-NeXT-7B} &
  \textbf{LLaVA-NeXT-13B} \\
\midrule
77.8\% & (0.5,\ 0.5) & (0.5,\ 0.4) & (0.5,\ 0.5) & (0.5,\ 0.5) \\
88.9\% & (0.5,\ 1.0) & (0.5,\ 0.4) & (0.5,\ 0.4) & (0.5,\ 0.5) \\
94.4\% & (0.5,\ 1.0) & (0.5,\ 0.4) & (0.5,\ 0.5) & (0.5,\ 0.5) \\
\bottomrule
\end{tabular}
\end{table*}

\vspace{-2mm}

\begin{table*}[!t]
\centering
\caption{Hyperparameter settings ($\alpha$, $\lambda$) for Qwen2.5-VL-7B, InternVL3-8B, and LLaVA-Video.}
\label{tab:hyperparams_advanced}
\vspace{-2mm}

\setlength{\tabcolsep}{3pt}
\renewcommand{\arraystretch}{0.88}

\resizebox{0.78\textwidth}{!}{%
\begin{tabular}{@{}ccc@{}}
\begin{tabular}{@{}lc@{}}
\multicolumn{2}{c}{\textbf{Qwen2.5-VL-7B}} \\[2pt]
\toprule
\textbf{Pruning Ratio} & ($\alpha$, $\lambda$) \\
\midrule
80.2\% & (0.5, 0.4) \\
90.1\% & (0.5, 0.4) \\
\bottomrule
\end{tabular}
&
\begin{tabular}{@{}lc@{}}
\multicolumn{2}{c}{\textbf{InternVL3-8B}} \\[2pt]
\toprule
\textbf{Pruning Ratio} & ($\alpha$, $\lambda$) \\
\midrule
80.0\% & (0.2, 0.2) \\
90.0\% & (0.3, 0.2) \\
\bottomrule
\end{tabular}
&
\begin{tabular}{@{}lc@{}}
\multicolumn{2}{c}{\textbf{LLaVA-Video-7B}} \\[2pt]
\toprule
\textbf{Pruning Ratio} & ($\alpha$, $\lambda$) \\
\midrule
62.1\% & (0.5, 0.5) \\
81.1\% & (0.5, 0.5) \\
90.5\% & (0.5, 0.5) \\
\bottomrule
\end{tabular}
\end{tabular}
}

\vspace{-3mm}
\end{table*}

\subsection{Pruning Schedule}
\label{sec:appendix_implementation_schedule}

For all models, Stage~I applies TOPS immediately after the multimodal projector to reduce the initial token count before entering the LLM; Stage~II then applies two successive TOPS passes at designated LLM layers to reach the final budget.

\begin{table*}[h]
\centering
\caption{TOPS pruning schedule for LLaVA-1.5 and LLaVA-NeXT at three token budgets (T). Stage~I reduces visual tokens before the LLM to $2T$; Stage~II applies two successive TOPS passes at designated LLM layers to reach the final token count.}
\label{tab:pruning_schedule}
\setlength{\tabcolsep}{5pt}
\small
\begin{tabular}{l c c c c}
\toprule
\textbf{Model} & \textbf{Target} $T$ & \textbf{Stage~I} & \textbf{Stage~II Layers} & \textbf{Stage~II Budgets} \\
\midrule
\multirow{3}{*}{LLaVA-1.5-7B}
  & 128 & $576 \!\to\! 256$ & (L12,\ L24) & $(256\!\to\!128,\ 128\!\to\!32)$ \\
  & 64  & $576 \!\to\! 128$ & (L12,\ L24) & $(128\!\to\!64,\;\;\, 64\!\to\!16)$  \\
  & 32  & $576 \!\to\! 64$  & (L12,\ L24) & $(64\!\to\!32,\;\;\;\, 32\!\to\!8)$   \\
\midrule
\multirow{3}{*}{LLaVA-1.5-13B}
  & 128 & $576 \!\to\! 256$ & (L15,\ L30) & $(256\!\to\!128,\ 128\!\to\!32)$ \\
  & 64  & $576 \!\to\! 128$ & (L15,\ L30) & $(128\!\to\!64,\;\;\, 64\!\to\!16)$  \\
  & 32  & $576 \!\to\! 64$  & (L15,\ L30) & $(64\!\to\!32,\;\;\;\, 32\!\to\!8)$   \\
\midrule
\multirow{3}{*}{LLaVA-NeXT-7B}
  & 640 & $2880 \!\to\! 1280$ & (L12,\ L24) & $(1280\!\to\!640,\ 640\!\to\!160)$ \\
  & 320 & $2880 \!\to\! 640$  & (L12,\ L24) & $(640\!\to\!320,\ 320\!\to\!80)$   \\
  & 160 & $2880 \!\to\! 320$  & (L12,\ L24) & $(320\!\to\!160,\ 160\!\to\!40)$   \\
\midrule
\multirow{3}{*}{LLaVA-NeXT-13B}
  & 640 & $2880 \!\to\! 1280$ & (L15,\ L30) & $(1280\!\to\!640,\ 640\!\to\!160)$ \\
  & 320 & $2880 \!\to\! 640$  & (L15,\ L30) & $(640\!\to\!320,\ 320\!\to\!80)$   \\
  & 160 & $2880 \!\to\! 320$  & (L15,\ L30) & $(320\!\to\!160,\ 160\!\to\!40)$   \\
\bottomrule
\end{tabular}
\end{table*}

\vspace{-2mm}

\begin{table*}[h]
\centering
\caption{TOPS pruning schedule for Qwen2.5-VL-7B (initial: 1296 tokens) and InternVL3-8B (initial: 1280 tokens). Stage~I reduces tokens before the LLM; Stage~II applies two successive TOPS passes at designated LLM layers.}
\label{tab:pruning_schedule_advanced}
\setlength{\tabcolsep}{5pt}
\small
\begin{tabular}{l c c c c}
\toprule
\textbf{Model} & \textbf{Stage~I} & \textbf{Stage~II Layers} & \textbf{Stage~II Budgets} \\
\midrule
\multirow{3}{*}{Qwen2.5-VL-7B}
  & $1296 \!\to\! 512$ & (L12,\ L16) & $(512\!\to\!281,\ 281\!\to\!77)$  \\
  & $1296 \!\to\! 256$ & (L12,\ L16) & $(256\!\to\!139,\ 139\!\to\!39)$  \\
  & $1296 \!\to\! 128$ & (L12,\ L16) & $(128\!\to\!71,\;\;\, 71\!\to\!19)$   \\
\midrule
\multirow{3}{*}{InternVL3-8B}
  & $1280 \!\to\! 512$ & (L12,\ L16) & $(512\!\to\!281,\ 281\!\to\!77)$  \\
  & $1280 \!\to\! 256$ & (L12,\ L16) & $(256\!\to\!139,\ 139\!\to\!39)$  \\
  & $1280 \!\to\! 128$ & (L12,\ L16) & $(128\!\to\!71,\;\;\, 71\!\to\!19)$   \\
\bottomrule
\end{tabular}
\end{table*}

\subsection{Hardware and Evaluation Protocol}
\label{sec:appendix_implementation_hardware}

All experiments are conducted on NVIDIA A800-80GB GPUs. Inference is performed with batch size 1 to ensure fair latency comparison across methods. We follow the standard evaluation protocol for each benchmark, using greedy decoding without sampling for generative tasks.

\begin{algorithm}[H]
\caption{TOPS Stage~II --- In-LLM Layer-Wise Token Pruning}
\label{alg:stage2}
\scriptsize
\begin{algorithmic}[1]
\REQUIRE
\begin{tabular}[t]{@{}l@{}}
Hidden states $\{h_i^{(l)}\}$ at pruning layer $l$, \\
attention weights $A^{(l)}$, current visual tokens $V^{(l)}$, \\
target budget $M_l$, balance factors $\alpha_2,\lambda_2$
\end{tabular}
\ENSURE Updated visual token set $V^{(l+1)}$

\STATE
\begin{tabular}[t]{@{}l@{}}
Compute dynamic text rater set $\mathcal{Q}^{(l)}$ \\
via text-visual relevance thresholding
\end{tabular}

\STATE
\begin{tabular}[t]{@{}l@{}}
Compute text-guided relevance $r_i^{(l)}$ for all $i\in V^{(l)}$ \\
using $\mathcal{Q}^{(l)}$
\end{tabular}

\STATE
\begin{tabular}[t]{@{}l@{}}
Precompute similarity matrix $\mathbf{F}^{(l)}$ from \\
$\{h_i^{(l)}\}_{i\in V^{(l)}}$
\end{tabular}

\STATE
\begin{tabular}[t]{@{}l@{}}
$S \leftarrow \{i_0\}$ where $i_0=\arg\max_i r_i^{(l)}$ \\
\COMMENT{seed: highest task relevance}
\end{tabular}

\STATE
\begin{tabular}[t]{@{}l@{}}
Initialize $\mathrm{maxsim}_j \leftarrow \mathbf{F}^{(l)}[j,i_0]$ \\
for all $j\notin S$
\end{tabular}

\WHILE{$|S| < M_l$}
    \FOR{each token $i\in V^{(l)}\setminus S$}
        \STATE $\mathrm{div}_i^{(l)} \leftarrow 1-\max_{j\in S}\mathbf{F}^{(l)}[i,j]$

        \STATE
        \begin{tabular}[t]{@{}l@{}}
        $\mathrm{cov}_i^{(l)} \leftarrow \displaystyle\sum_{j\in V^{(l)}\setminus S}$ \\
        $\qquad\max\!\big(0,\,\mathbf{F}^{(l)}[i,j]
        -\max_{k\in S}\mathbf{F}^{(l)}[j,k]\big)$
        \end{tabular}
    \ENDFOR

    \STATE
    \begin{tabular}[t]{@{}l@{}}
    Normalize each criterion by its mean: \\
    $\tilde{r}_i^{(l)},\;\widetilde{\mathrm{div}}_i^{(l)},\;
    \widetilde{\mathrm{cov}}_i^{(l)}$
    \end{tabular}

    \STATE
    \begin{tabular}[t]{@{}l@{}}
    $\mathrm{score}_i^{(l)} \leftarrow
    \tilde{r}_i^{(l)}+\alpha_2\,\widetilde{\mathrm{div}}_i^{(l)}$ \\
    $\qquad+\lambda_2\,\widetilde{\mathrm{cov}}_i^{(l)}$
    \end{tabular}

    \STATE $i^\star \leftarrow \arg\max_{i\in V^{(l)}\setminus S}\mathrm{score}_i^{(l)}$
    \STATE $S \leftarrow S\cup\{i^\star\}$

    \STATE
    \begin{tabular}[t]{@{}l@{}}
    Update $\mathrm{maxsim}_j \leftarrow
    \max(\mathrm{maxsim}_j,\mathbf{F}^{(l)}[j,i^\star])$ \\
    for all $j\notin S$
    \end{tabular}
\ENDWHILE

\STATE
\begin{tabular}[t]{@{}l@{}}
Rebuild hidden sequence: \\
$H^{(l)} \leftarrow [H_{\text{sys}};\;\{h_i^{(l)}\}_{i\in S};\;H_{\text{text}}]$
\end{tabular}

\STATE
\begin{tabular}[t]{@{}l@{}}
Rebuild \texttt{attention\_mask}: set 1 for retained positions \\
and 0 for pruned visual positions
\end{tabular}

\STATE
\begin{tabular}[t]{@{}l@{}}
Rebuild \texttt{position\_ids}: re-index retained positions \\
contiguously from 0
\end{tabular}

\STATE $V^{(l+1)} \leftarrow S$
\STATE \textbf{return} $V^{(l+1)}$
\end{algorithmic}
\end{algorithm}

\clearpage

\begin{table*}[!t]
    \centering
    \centering
\begingroup
\setlength{\tabcolsep}{1.6pt}
\renewcommand{\arraystretch}{0.88}
\caption{Performance comparison of different pruning methods on LLaVA-1.5-7B. \textbf{Rel.} denotes the ratio of pruned accuracy to baseline accuracy. Red: attention-based; Green: attention\&diversity; Blue: diversity-based; Cyan: coverage-based; Purple: ours.}
\label{tab:pruning_comparison_7b}
\vspace{-1mm}
\resizebox{0.97\textwidth}{!}{%
\begin{tabular}{l|cccccccc|cc}
\toprule
\textbf{Method} & \textbf{GQA} & \textbf{SQA}$^{IMG}$ & \textbf{VQA}$^{\text{Text}}$ & \textbf{POPE} & \textbf{MME} & \textbf{MMB}$^{\text{EN}}$ & \textbf{MMB}$^{\text{CN}}$ & \textbf{MMVet} & \textbf{Acc.} & \textbf{Rel.} \\
\midrule
\multicolumn{11}{c}{\textit{Upper Bound: All 576 tokens (100\%)}} \\
\midrule
Baseline & 61.9 & 69.5 & 58.2 & 85.9 & 1506.5 & 64.7 & 58.1 & 31.3 & 63.1 & 100.0\% \\
\midrule
\multicolumn{11}{c}{\textit{Retain 128 Tokens (\(\downarrow\) 77.8\%)}} \\
\midrule
\rowcolor[rgb]{1.0,0.88,0.88}
FastV (ECCV24) & 54.0 & \underline{69.2} & 56.4 & 68.2 & 1368.9 & \textbf{63.0} & 55.9 & 27.0 & 57.8 & 91.6\% \\
\rowcolor[rgb]{1.0,0.88,0.88}
PDrop (CVPR25) & 57.1 & \textbf{70.1} & 56.7 & 77.5 & \underline{1444.1} & 62.3 & 55.3 & 27.6 & 59.9 & 94.9\% \\
\rowcolor[rgb]{1.0,0.88,0.88}
SparseVLM (ICML25) & 57.3 & 69.0 & 56.3 & 83.1 & 1399.3 & 62.6 & 56.9 & 29.7 & 60.6 & 96.0\% \\
\rowcolor[rgb]{0.88,0.96,0.88}
PruMerge+ (ICCV25) & 58.2 & 69.1 & 54.0 & 83.1 & 1408.1 & 61.8 & 55.8 & 30.4 & 60.4 & 95.7\% \\
\rowcolor[rgb]{0.88,0.96,0.88}
TRIM (COLING25) & 58.4 & 68.6 & 52.2 & 85.3 & 1413.4 & \textbf{63.0} & 52.3 & 29.9 & 60.1 & 95.2\% \\
\rowcolor[rgb]{0.88,0.96,0.88}
VisionZip (CVPR25) & 57.6 & 68.7 & 56.9 & 83.3 & 1436.9 & 62.1 & 57.0 & \textbf{31.6} & 61.1 & 96.8\% \\
\rowcolor[rgb]{0.88,0.93,1.0}
DART (EMNLP25) & 57.9 & 69.1 & 56.3 & 80.4 & 1408.7 & 60.7 & \textbf{57.3} & 30.9 & 60.4 & 95.7\% \\
\rowcolor[rgb]{0.88,0.93,1.0}
DivPrune (CVPR25) & \underline{59.4} & 68.6 & 55.9 & \textbf{87.0} & 1405.1 & 61.5 & 54.8 & 30.6 & 61.0 & 96.7\% \\
\rowcolor[rgb]{0.88,0.97,0.97}
SCOPE (NeurIPS25) & \underline{59.4} & 68.5 & \textbf{57.1} & 85.9 & 1440.5 & \underline{62.7} & 57.0 & \underline{31.3} & \underline{61.7} & \underline{97.8\%} \\
\rowcolor[rgb]{0.93,0.88,1.0}
\textbf{TOPS (Ours)} & \textbf{60.5} & 68.2 & \underline{57.0} & \underline{86.8} & \textbf{1482.7} & 62.5 & \underline{57.2} & 30.0 & \textbf{62.0} & \textbf{98.3\%} \\
\midrule
\multicolumn{11}{c}{\textit{Retain 64 Tokens (\(\downarrow\) 88.9\%)}} \\
\midrule
\rowcolor[rgb]{1.0,0.88,0.88}
FastV (ECCV24) & 46.0 & \textbf{70.1} & 51.6 & 35.5 & 973.5 & 50.1 & 42.1 & 18.9 & 45.4 & 71.9\% \\
\rowcolor[rgb]{1.0,0.88,0.88}
PDrop (CVPR25) & 46.1 & 68.8 & 49.2 & 40.8 & 982.2 & 48.0 & 36.6 & 17.7 & 44.5 & 70.5\% \\
\rowcolor[rgb]{1.0,0.88,0.88}
SparseVLM (ICML25) & 52.0 & 69.2 & 52.1 & 69.7 & 1190.4 & 58.3 & 49.6 & 24.4 & 54.4 & 86.2\% \\
\rowcolor[rgb]{0.88,0.96,0.88}
PruMerge+ (ICCV25) & 55.4 & \underline{69.5} & 52.0 & 75.7 & 1316.8 & 59.6 & 52.1 & 28.0 & 57.3 & 90.8\% \\
\rowcolor[rgb]{0.88,0.96,0.88}
TRIM (COLING25) & 56.6 & 69.0 & 49.7 & \underline{85.9} & 1350.9 & \underline{60.9} & 48.2 & 24.8 & 57.8 & 91.6\% \\
\rowcolor[rgb]{0.88,0.96,0.88}
VisionZip (CVPR25) & 55.1 & 69.0 & 55.5 & 77.0 & 1365.2 & 60.1 & 55.4 & 29.4 & 58.7 & 93.0\% \\
\rowcolor[rgb]{0.88,0.93,1.0}
DART (EMNLP25) & 54.7 & 69.3 & 54.7 & 73.8 & 1365.1 & 59.5 & 54.0 & 26.5 & 57.6 & 91.3\% \\
\rowcolor[rgb]{0.88,0.93,1.0}
DivPrune (CVPR25) & 57.5 & 68.0 & 54.5 & 85.5 & 1334.7 & 60.1 & 52.3 & 28.1 & 59.1 & 93.7\% \\
\rowcolor[rgb]{0.88,0.97,0.97}
SCOPE (NeurIPS25) & \underline{58.3} & 68.7 & \textbf{56.5} & 84.1 & \underline{1399.6} & \textbf{61.0} & \underline{56.0} & \underline{30.5} & \underline{60.6} & \underline{96.0\%} \\
\rowcolor[rgb]{0.93,0.88,1.0}
\textbf{TOPS (Ours)} & \textbf{58.7} & 68.6 & \underline{56.2} & \textbf{86.5} & \textbf{1442.7} & \underline{60.9} & \textbf{56.5} & \textbf{30.6} & \textbf{61.3} & \textbf{97.1\%} \\
\midrule
\multicolumn{11}{c}{\textit{Retain 32 Tokens (\(\downarrow\) 94.4\%)}} \\
\midrule
\rowcolor[rgb]{0.88,0.96,0.88}
PruMerge+ (ICCV25) & 52.9 & 67.9 & 49.2 & 66.7 & 1236.6 & 55.1 & 45.9 & 24.7 & 53.0 & 84.0\% \\
\rowcolor[rgb]{0.88,0.96,0.88}
TRIM (COLING25) & 54.5 & 68.1 & 47.6 & \textbf{84.9} & 1251.8 & 57.7 & 40.1 & 20.5 & 54.5 & 86.4\% \\
\rowcolor[rgb]{0.88,0.96,0.88}
VisionZip (CVPR25) & 51.8 & 69.1 & 53.1 & 69.4 & 1251.2 & 57.0 & 50.3 & 25.3 & 54.8 & 86.8\% \\
\rowcolor[rgb]{0.88,0.93,1.0}
DART (EMNLP25) & 52.9 & \underline{69.3} & 52.2 & 69.1 & 1273.3 & 58.5 & 50.0 & 25.0 & 55.1 & 87.3\% \\
\rowcolor[rgb]{0.88,0.93,1.0}
DivPrune (CVPR25) & 54.9 & 68.6 & 52.9 & 81.5 & 1284.9 & 57.6 & 49.1 & 26.3 & 56.9 & 90.2\% \\
\rowcolor[rgb]{0.88,0.97,0.97}
SCOPE (NeurIPS25) & \underline{56.2} & \textbf{69.4} & \underline{54.8} & 80.2 & \underline{1371.6} & \textbf{60.7} & \underline{52.5} & \textbf{29.8} & \underline{59.0} & \underline{93.5\%} \\
\rowcolor[rgb]{0.93,0.88,1.0}
\textbf{TOPS (Ours)} & \textbf{56.7} & 68.8 & \textbf{54.9} & \underline{83.5} & \textbf{1384.7} & \underline{59.5} & \textbf{55.1} & \underline{29.7} & \textbf{59.7} & \textbf{94.6\%} \\
\bottomrule
\end{tabular}
}
\endgroup
    \vspace{-3mm}
\end{table*}

\section{Experiments on More MLLMs}
\label{sec:appendix_13b}

To verify that TOPS generalizes across model scales, we additionally evaluate it on LLaVA-1.5-7B/13B and LLaVA-NeXT-7B/13B. As shown in Tables~\ref{tab:pruning_comparison_7b}--\ref{tab:pruning_comparison_next13b}, TOPS consistently outperforms all baselines across three compression levels.

\begin{table*}[!t]
    \centering
    \centering
\setlength{\tabcolsep}{2pt}
\caption{Performance comparison of different pruning methods on LLaVA-NeXT-7B. \textbf{Rel.} represents the ratio of pruned model's Acc. to the baseline's Acc. \colorbox[rgb]{1.0,0.88,0.88}{Red}: attention-based. \colorbox[rgb]{0.88,0.96,0.88}{Green}: attention\&diversity. \colorbox[rgb]{0.88,0.93,1.0}{Blue}: diversity-based. \colorbox[rgb]{0.88,0.97,0.97}{Cyan}: coverage-based. \colorbox[rgb]{0.93,0.88,1.0}{Purple}: ours.}
\label{tab:pruning_comparison_next}
\resizebox{\textwidth}{!}{%
\begin{tabular}{l|cccccccc|cc}
\toprule
\textbf{Method} & \textbf{GQA} & \textbf{SQA}$^{IMG}$ & \textbf{VQA}$^{\text{Text}}$ & \textbf{POPE} & \textbf{MME} & \textbf{MMB}$^{\text{EN}}$ & \textbf{MMB}$^{\text{CN}}$ & \textbf{MMVet} & \textbf{Acc.} & \textbf{Rel.} \\
\midrule
\multicolumn{11}{c}{\textit{Upper Bound: All 2880 tokens (100\%)}} \\
\midrule
Baseline & 62.5 & 67.5 & 60.3 & 86.8 & 1511.8 & 65.8 & 57.3 & 40.0 & 64.5 & 100.0\% \\
\midrule
\multicolumn{11}{c}{\textit{Retain 640 Tokens (\(\downarrow\) 77.8\%)}} \\
\midrule
\rowcolor[rgb]{1.0,0.88,0.88}
FastV (ECCV24) & 58.9 & 67.4 & 58.1 & 79.5 & 1412.6 & 63.1 & 53.5 & \underline{39.5} & 61.3 & 95.0\% \\
\rowcolor[rgb]{1.0,0.88,0.88}
PDrop (CVPR25) & 60.0 & 66.7 & 57.8 & 83.8 & 1475.9 & 64.1 & 55.2 & 36.7 & 62.3 & 96.6\% \\
\rowcolor[rgb]{1.0,0.88,0.88}
SparseVLM (ICML25) & 61.2 & 67.6 & 59.7 & 85.3 & 1456.8 & 65.9 & \underline{58.6} & 36.1 & 63.4 & 98.3\% \\
\rowcolor[rgb]{0.88,0.96,0.88}
PruMerge+ (ICCV25) & 60.8 & 67.8 & 54.9 & 85.3 & 1480.2 & 64.6 & 57.3 & 32.7 & 62.2 & 96.4\% \\
\rowcolor[rgb]{0.88,0.96,0.88}
TRIM (COLING25) & \textbf{62.1} & 66.9 & 54.8 & \underline{86.9} & 1471.8 & \textbf{66.8} & 55.8 & 37.8 & 63.1 & 97.8\% \\
\rowcolor[rgb]{0.88,0.96,0.88}
VisionZip (CVPR25) & 61.2 & \underline{68.1} & 59.9 & 86.0 & \underline{1493.4} & 65.8 & 58.1 & 38.9 & 64.1 & 99.4\% \\
\rowcolor[rgb]{0.88,0.93,1.0}
DART (EMNLP25) & 61.3 & 68.0 & 59.5 & 85.0 & 1450.2 & 64.9 & 57.1 & 36.9 & 63.2 & 98.0\% \\
\rowcolor[rgb]{0.88,0.93,1.0}
DivPrune (CVPR25) & 61.9 & 67.8 & 57.0 & \underline{86.9} & 1469.7 & 65.8 & 57.3 & 38.0 & 63.5 & 98.4\% \\
\rowcolor[rgb]{0.88,0.97,0.97}
SCOPE (NeurIPS25) & \underline{62.0} & 68.0 & \textbf{60.1} & 86.7 & 1485.1 & \underline{66.2} & 58.2 & \textbf{39.7} & \underline{64.4} & \underline{99.8\%} \\
\rowcolor[rgb]{0.93,0.88,1.0}
\textbf{TOPS (Ours)} & \underline{62.0} & \textbf{69.3} & \underline{60.0} & \textbf{87.6} & \textbf{1527.6} & 65.7 & \textbf{59.0} & 36.4 & \textbf{64.5} & \textbf{100.0\%} \\
\midrule
\multicolumn{11}{c}{\textit{Retain 320 Tokens (\(\downarrow\) 88.9\%)}} \\
\midrule
\rowcolor[rgb]{1.0,0.88,0.88}
FastV (ECCV24) & 49.8 & 66.6 & 52.2 & 49.5 & 1099.0 & 53.4 & 42.5 & 20.0 & 48.6 & 75.3\% \\
\rowcolor[rgb]{1.0,0.88,0.88}
PDrop (CVPR25) & 50.4 & 66.7 & 49.0 & 60.8 & 1171.5 & 55.5 & 44.7 & 24.0 & 51.2 & 79.4\% \\
\rowcolor[rgb]{1.0,0.88,0.88}
SparseVLM (ICML25) & 57.9 & 67.2 & 56.5 & 76.9 & 1386.1 & 63.1 & 56.7 & 32.8 & 60.1 & 93.2\% \\
\rowcolor[rgb]{0.88,0.96,0.88}
PruMerge+ (ICCV25) & 58.8 & \underline{68.1} & 54.0 & 79.5 & 1444.3 & 63.0 & 55.6 & 31.4 & 60.3 & 93.5\% \\
\rowcolor[rgb]{0.88,0.96,0.88}
TRIM (COLING25) & 59.9 & 66.2 & 50.2 & \textbf{86.5} & 1443.8 & 63.5 & 51.0 & 32.7 & 60.3 & 93.5\% \\
\rowcolor[rgb]{0.88,0.96,0.88}
VisionZip (CVPR25) & 58.9 & 67.5 & \textbf{58.8} & 82.3 & 1397.1 & 63.3 & 55.6 & 35.8 & 61.5 & 95.3\% \\
\rowcolor[rgb]{0.88,0.93,1.0}
DART (EMNLP25) & 59.5 & 67.5 & 57.6 & 81.0 & 1419.5 & 64.2 & 55.7 & 35.7 & 61.5 & 95.3\% \\
\rowcolor[rgb]{0.88,0.93,1.0}
DivPrune (CVPR25) & \underline{61.1} & 67.7 & 56.2 & 84.7 & 1423.3 & 63.9 & 55.7 & 34.8 & 61.9 & 96.0\% \\
\rowcolor[rgb]{0.88,0.97,0.97}
SCOPE (NeurIPS25) & 60.9 & 68.0 & 58.3 & 85.0 & \underline{1477.0} & \underline{65.0} & \underline{57.6} & \underline{36.1} & \underline{63.1} & \underline{97.8\%} \\
\rowcolor[rgb]{0.93,0.88,1.0}
\textbf{TOPS (Ours)} & \textbf{61.2} & \textbf{68.5} & \underline{58.4} & \underline{86.3} & \textbf{1500.1} & \textbf{65.8} & \textbf{58.4} & \textbf{37.4} & \textbf{63.9} & \textbf{99.1\%} \\
\midrule
\multicolumn{11}{c}{\textit{Retain 160 Tokens (\(\downarrow\) 94.4\%)}} \\
\midrule
\rowcolor[rgb]{0.88,0.96,0.88}
PruMerge+ (ICCV25) & 56.2 & 66.9 & 50.3 & 71.1 & 1289.6 & 58.0 & 48.9 & 29.3 & 55.6 & 86.2\% \\
\rowcolor[rgb]{0.88,0.96,0.88}
TRIM (COLING25) & 57.4 & 65.5 & 45.8 & \textbf{84.8} & 1275.8 & 61.6 & 45.2 & 29.6 & 56.7 & 87.9\% \\
\rowcolor[rgb]{0.88,0.96,0.88}
VisionZip (CVPR25) & 55.2 & \textbf{67.9} & 55.0 & 74.9 & 1327.8 & 58.6 & 50.4 & 32.3 & 57.6 & 89.3\% \\
\rowcolor[rgb]{0.88,0.93,1.0}
DART (EMNLP25) & 56.8 & \underline{67.8} & 54.9 & 75.3 & 1325.4 & 62.0 & 53.6 & 32.2 & 58.6 & 90.9\% \\
\rowcolor[rgb]{0.88,0.93,1.0}
DivPrune (CVPR25) & 59.3 & 67.1 & 54.1 & 80.0 & 1356.6 & 62.9 & 53.7 & 32.0 & 59.6 & 92.4\% \\
\rowcolor[rgb]{0.88,0.97,0.97}
SCOPE (NeurIPS25) & \textbf{59.8} & 67.1 & \textbf{56.8} & 81.3 & \underline{1402.2} & \underline{63.3} & \underline{56.4} & \underline{32.4} & \underline{60.9} & \underline{94.4\%} \\
\rowcolor[rgb]{0.93,0.88,1.0}
\textbf{TOPS (Ours)} & \underline{59.6} & 67.7 & \underline{56.7} & \underline{83.6} & \textbf{1446.8} & \textbf{64.2} & \textbf{57.1} & \textbf{36.5} & \textbf{62.2} & \textbf{96.4\%} \\
\bottomrule
\end{tabular}
}

%
%

    \vspace{-3mm}
\end{table*}

\begin{table*}[!t]
    \centering
    \centering
\setlength{\tabcolsep}{2pt}
\caption{Performance comparison of different pruning methods on LLaVA-1.5-13B. \textbf{Rel.} represents the ratio of pruned model's Acc. to the baseline's Acc. \colorbox[rgb]{1.0,0.88,0.88}{Red}: attention-based. \colorbox[rgb]{0.88,0.96,0.88}{Green}: attention\&diversity. \colorbox[rgb]{0.88,0.93,1.0}{Blue}: diversity-based. \colorbox[rgb]{0.88,0.97,0.97}{Cyan}: coverage-based. \colorbox[rgb]{0.93,0.88,1.0}{Purple}: ours.}
\label{tab:pruning_comparison_13b}
\resizebox{\textwidth}{!}{%
\begin{tabular}{l|cccccccc|cc}
\toprule
\textbf{Method} & \textbf{GQA} & \textbf{SQA}$^{IMG}$ & \textbf{VQA}$^{\text{Text}}$ & \textbf{POPE} & \textbf{MME} & \textbf{MMB}$^{\text{EN}}$ & \textbf{MMB}$^{\text{CN}}$ & \textbf{MMVet} & \textbf{Acc.} & \textbf{Rel.} \\
\midrule
\multicolumn{11}{c}{\textit{Upper Bound: All 576 tokens (100\%)}} \\
\midrule
Baseline & 63.3 & 72.8 & 61.2 & 86.0 & 1531.2 & 68.5 & 63.5 & 36.2 & 66.0 & 100.0\% \\
\midrule
\multicolumn{11}{c}{\textit{Retain 128 Tokens (\(\downarrow\) 77.8\%)}} \\
\midrule
\rowcolor[rgb]{1.0,0.88,0.88}
FastV (ECCV24) & 58.3 & \underline{74.2} & 58.6 & 75.5 & 1460.6 & 66.1 & 62.3 & 32.8 & 62.6 & 94.8\% \\
\rowcolor[rgb]{1.0,0.88,0.88}
PDrop (CVPR25) & \textbf{61.0} & 73.3 & \textbf{60.2} & 83.6 & \underline{1489.5} & \underline{67.5} & \underline{62.8} & 35.1 & 64.4 & 97.6\% \\
\rowcolor[rgb]{1.0,0.88,0.88}
SparseVLM (ICML25) & 59.6 & \textbf{74.3} & 59.3 & 85.0 & 1487.9 & \textbf{68.4} & 62.6 & 35.2 & \underline{64.8} & \underline{98.2\%} \\
\rowcolor[rgb]{0.88,0.96,0.88}
PruMerge+ (ICCV25) & 58.3 & 73.3 & 56.1 & 82.7 & 1445.9 & 66.3 & 61.2 & 33.6 & 63.0 & 95.4\% \\
\rowcolor[rgb]{0.88,0.96,0.88}
TRIM (COLING25) & 59.4 & 72.4 & 55.0 & \textbf{86.8} & 1426.9 & 67.1 & 58.4 & 35.1 & 63.2 & 95.7\% \\
\rowcolor[rgb]{0.88,0.96,0.88}
VisionZip (CVPR25) & 57.9 & 73.8 & 58.9 & 82.7 & 1449.2 & 67.4 & 62.5 & 36.0 & 64.0 & 96.9\% \\
\rowcolor[rgb]{0.88,0.93,1.0}
DART (EMNLP25) & 57.7 & \underline{74.2} & 58.7 & 80.4 & 1395.0 & 65.4 & 62.2 & 34.8 & 62.9 & 95.3\% \\
\rowcolor[rgb]{0.88,0.93,1.0}
DivPrune (CVPR25) & 59.2 & 72.8 & 58.0 & \textbf{86.8} & 1457.7 & 66.3 & 60.7 & 34.4 & 63.9 & 96.8\% \\
\rowcolor[rgb]{0.88,0.97,0.97}
SCOPE (NeurIPS25) & 59.1 & 73.8 & 58.9 & 86.0 & 1440.2 & 66.7 & \textbf{62.9} & \underline{36.6} & 64.5 & 97.7\% \\
\rowcolor[rgb]{0.93,0.88,1.0}
\textbf{TOPS (Ours)} & \underline{60.2} & 72.9 & \underline{59.4} & \underline{86.6} & \textbf{1522.8} & 66.6 & 62.4 & \textbf{38.1} & \textbf{65.3} & \textbf{98.9\%} \\
\midrule
\multicolumn{11}{c}{\textit{Retain 64 Tokens (\(\downarrow\) 88.9\%)}} \\
\midrule
\rowcolor[rgb]{1.0,0.88,0.88}
FastV (ECCV24) & 51.9 & 73.1 & 53.4 & 56.9 & 1246.4 & 59.2 & 55.1 & 26.9 & 54.9 & 83.2\% \\
\rowcolor[rgb]{1.0,0.88,0.88}
PDrop (CVPR25) & 54.1 & 73.1 & 55.3 & 66.1 & 1247.0 & 63.1 & 56.6 & 21.9 & 56.6 & 85.8\% \\
\rowcolor[rgb]{1.0,0.88,0.88}
SparseVLM (ICML25) & 55.9 & 73.0 & 57.1 & 77.9 & 1374.3 & 65.2 & 60.3 & 32.9 & 61.4 & 93.0\% \\
\rowcolor[rgb]{0.88,0.96,0.88}
PruMerge+ (ICCV25) & 56.3 & 73.5 & 54.4 & 75.7 & 1338.2 & 65.0 & 59.3 & 30.3 & 60.2 & 91.2\% \\
\rowcolor[rgb]{0.88,0.96,0.88}
TRIM (COLING25) & \underline{57.9} & 72.0 & 52.0 & \textbf{86.5} & 1406.2 & 65.0 & 52.7 & 27.8 & 60.5 & 91.7\% \\
\rowcolor[rgb]{0.88,0.96,0.88}
VisionZip (CVPR25) & 56.2 & \textbf{74.2} & \underline{57.4} & 75.7 & 1379.6 & 64.9 & 61.3 & 33.4 & 61.5 & 93.2\% \\
\rowcolor[rgb]{0.88,0.93,1.0}
DART (EMNLP25) & 55.7 & \underline{73.8} & \underline{57.4} & 72.8 & 1380.0 & 64.7 & 60.6 & 32.8 & 60.8 & 92.1\% \\
\rowcolor[rgb]{0.88,0.93,1.0}
DivPrune (CVPR25) & \underline{57.9} & 71.7 & 57.3 & 84.5 & \underline{1454.2} & 64.1 & 59.8 & 29.3 & 62.2 & 94.2\% \\
\rowcolor[rgb]{0.88,0.97,0.97}
SCOPE (NeurIPS25) & \textbf{58.6} & 73.6 & \textbf{58.2} & 83.1 & 1445.2 & \underline{65.9} & \textbf{62.6} & \underline{34.5} & \underline{63.6} & \underline{96.4\%} \\
\rowcolor[rgb]{0.93,0.88,1.0}
\textbf{TOPS (Ours)} & \textbf{58.6} & \underline{73.8} & 56.6 & \underline{85.4} & \textbf{1462.2} & \textbf{66.8} & \underline{62.3} & \textbf{37.0} & \textbf{64.2} & \textbf{97.3\%} \\
\midrule
\multicolumn{11}{c}{\textit{Retain 32 Tokens (\(\downarrow\) 94.4\%)}} \\
\midrule
\rowcolor[rgb]{0.88,0.96,0.88}
PruMerge+ (ICCV25) & 54.1 & 71.7 & 52.4 & 67.4 & 1269.1 & 61.1 & 53.5 & 28.7 & 56.5 & 85.6\% \\
\rowcolor[rgb]{0.88,0.96,0.88}
TRIM (COLING25) & 55.6 & 70.4 & 49.6 & \textbf{85.8} & 1284.7 & 63.1 & 45.4 & 26.4 & 57.6 & 87.3\% \\
\rowcolor[rgb]{0.88,0.96,0.88}
VisionZip (CVPR25) & 52.7 & 72.9 & 55.2 & 66.8 & 1257.7 & 61.2 & 55.8 & 29.3 & 57.1 & 86.5\% \\
\rowcolor[rgb]{0.88,0.93,1.0}
DART (EMNLP25) & 53.9 & \underline{73.2} & 55.1 & 66.9 & 1282.8 & 61.9 & 56.2 & 29.4 & 57.6 & 87.3\% \\
\rowcolor[rgb]{0.88,0.93,1.0}
DivPrune (CVPR25) & \underline{56.2} & 70.9 & 54.6 & \underline{79.3} & 1405.2 & 61.7 & 57.2 & 27.8 & 59.7 & 90.5\% \\
\rowcolor[rgb]{0.88,0.97,0.97}
SCOPE (NeurIPS25) & \textbf{57.2} & 72.4 & \underline{57.2} & 77.6 & \underline{1413.4} & \underline{63.5} & \underline{60.1} & \underline{34.0} & \underline{61.6} & \underline{93.3\%} \\
\rowcolor[rgb]{0.93,0.88,1.0}
\textbf{TOPS (Ours)} & 56.1 & \textbf{73.3} & \textbf{57.7} & 78.2 & \textbf{1442.2} & \textbf{65.2} & \textbf{61.9} & \textbf{35.1} & \textbf{62.5} & \textbf{94.7\%} \\
\bottomrule
\end{tabular}
}

%
%
%
    \vspace{-3mm}
\end{table*}

\begin{table*}[!t]
    \centering
    \centering
\setlength{\tabcolsep}{2pt}
\caption{Performance comparison of different pruning methods on LLaVA-NeXT-13B. \textbf{Rel.} represents the ratio of pruned model's Acc. to the baseline's Acc. \colorbox[rgb]{1.0,0.88,0.88}{Red}: attention-based. \colorbox[rgb]{0.88,0.96,0.88}{Green}: attention\&diversity. \colorbox[rgb]{0.88,0.93,1.0}{Blue}: diversity-based. \colorbox[rgb]{0.88,0.97,0.97}{Cyan}: coverage-based. \colorbox[rgb]{0.93,0.88,1.0}{Purple}: ours.}
\label{tab:pruning_comparison_next13b}
\resizebox{\textwidth}{!}{%
\begin{tabular}{l|cccccccc|cc}
\toprule
\textbf{Method} & \textbf{GQA} & \textbf{SQA}$^{IMG}$ & \textbf{VQA}$^{\text{Text}}$ & \textbf{POPE} & \textbf{MME} & \textbf{MMB}$^{\text{EN}}$ & \textbf{MMB}$^{\text{CN}}$ & \textbf{MMVet} & \textbf{Acc.} & \textbf{Rel.} \\
\midrule
\multicolumn{11}{c}{\textit{Upper Bound: All 2880 tokens (100\%)}} \\
\midrule
Baseline & 64.4 & 73.1 & 63.2 & 85.3 & 1539.5 & 68.5 & 61.2 & 45.0 & 67.2 & 100.0\% \\
\midrule
\multicolumn{11}{c}{\textit{Retain 640 Tokens (\(\downarrow\) 77.8\%)}} \\
\midrule
\rowcolor[rgb]{1.0,0.88,0.88}
FastV (ECCV24) & 60.9 & 71.7 & 60.7 & 80.2 & 1516.7 & 65.5 & 59.9 & 43.8 & 64.8 & 96.4\% \\
\rowcolor[rgb]{1.0,0.88,0.88}
PDrop (CVPR25) & 62.8 & 71.7 & 62.1 & 84.4 & 1559.1 & 66.6 & 60.8 & 39.7 & 65.8 & 97.9\% \\
\rowcolor[rgb]{1.0,0.88,0.88}
SparseVLM (ICML25) & 62.7 & \underline{72.5} & \textbf{62.8} & 85.6 & \underline{1562.7} & \underline{68.8} & \textbf{64.0} & 41.3 & 67.0 & 99.7\% \\
\rowcolor[rgb]{0.88,0.96,0.88}
PruMerge+ (ICCV25) & 62.8 & 70.6 & 56.2 & 83.7 & 1497.3 & 67.4 & 61.9 & 39.4 & 64.6 & 96.1\% \\
\rowcolor[rgb]{0.88,0.96,0.88}
TRIM (COLING25) & 63.1 & 71.2 & 57.6 & \textbf{87.3} & 1554.6 & 68.7 & 61.2 & 42.3 & 66.1 & 98.4\% \\
\rowcolor[rgb]{0.88,0.96,0.88}
VisionZip (CVPR25) & 62.9 & 70.8 & 62.1 & 85.8 & 1549.2 & 68.1 & 62.6 & \textbf{46.8} & \underline{67.1} & \underline{99.9\%} \\
\rowcolor[rgb]{0.88,0.93,1.0}
DART (EMNLP25) & 62.7 & 71.0 & 61.3 & 85.2 & 1542.4 & 67.6 & 61.9 & \underline{45.5} & 66.5 & 99.0\% \\
\rowcolor[rgb]{0.88,0.93,1.0}
DivPrune (CVPR25) & 63.5 & 72.2 & 59.2 & 86.5 & 1526.1 & 67.5 & 62.9 & 39.0 & 65.9 & 98.1\% \\
\rowcolor[rgb]{0.88,0.97,0.97}
SCOPE (NeurIPS25) & \underline{63.7} & 71.7 & 62.4 & 86.5 & \textbf{1573.1} & 67.6 & 63.2 & 40.9 & 66.8 & 99.4\% \\
\rowcolor[rgb]{0.93,0.88,1.0}
\textbf{TOPS (Ours)} & \textbf{64.1} & \textbf{72.8} & \underline{62.5} & \underline{86.7} & 1560.8 & \textbf{69.0} & \underline{63.4} & 44.4 & \textbf{67.6} & \textbf{100.6\%} \\
\midrule
\multicolumn{11}{c}{\textit{Retain 320 Tokens (\(\downarrow\) 88.9\%)}} \\
\midrule
\rowcolor[rgb]{1.0,0.88,0.88}
FastV (ECCV24) & 54.6 & 70.5 & 55.4 & 63.6 & 1279.0 & 59.8 & 54.4 & 30.2 & 56.6 & 84.2\% \\
\rowcolor[rgb]{1.0,0.88,0.88}
PDrop (CVPR25) & 57.7 & \underline{72.1} & 56.2 & 74.6 & 1386.3 & 62.8 & 55.3 & 29.5 & 59.7 & 88.8\% \\
\rowcolor[rgb]{1.0,0.88,0.88}
SparseVLM (ICML25) & 60.9 & 70.9 & 60.0 & 81.5 & 1491.6 & \textbf{68.0} & \textbf{63.5} & 39.3 & 64.8 & 96.4\% \\
\rowcolor[rgb]{0.88,0.96,0.88}
PruMerge+ (ICCV25) & 61.1 & 70.7 & 55.9 & 79.1 & 1426.5 & 66.6 & 60.6 & 36.5 & 62.7 & 93.3\% \\
\rowcolor[rgb]{0.88,0.96,0.88}
TRIM (COLING25) & 61.3 & 69.9 & 52.8 & \textbf{87.2} & 1476.6 & 67.3 & 57.4 & 33.1 & 62.9 & 93.6\% \\
\rowcolor[rgb]{0.88,0.96,0.88}
VisionZip (CVPR25) & 60.7 & 70.2 & 60.7 & 82.3 & 1487.3 & 66.5 & 62.3 & 41.1 & 64.8 & 96.4\% \\
\rowcolor[rgb]{0.88,0.93,1.0}
DART (EMNLP25) & 60.9 & 69.8 & 59.7 & 81.1 & 1457.4 & 65.9 & 61.9 & 41.4 & 64.2 & 95.5\% \\
\rowcolor[rgb]{0.88,0.93,1.0}
DivPrune (CVPR25) & 61.8 & \textbf{72.3} & 57.6 & 85.2 & 1473.0 & 65.9 & 61.9 & 39.2 & 64.7 & 96.3\% \\
\rowcolor[rgb]{0.88,0.97,0.97}
SCOPE (NeurIPS25) & \underline{62.7} & 71.0 & \underline{60.8} & 85.2 & \underline{1509.3} & 66.6 & \underline{63.1} & \textbf{42.8} & \underline{66.0} & \underline{98.2\%} \\
\rowcolor[rgb]{0.93,0.88,1.0}
\textbf{TOPS (Ours)} & \textbf{63.2} & 71.5 & \textbf{61.1} & \underline{85.9} & \textbf{1569.7} & \underline{67.6} & 63.0 & \underline{42.2} & \textbf{66.6} & \textbf{99.1\%} \\
\midrule
\multicolumn{11}{c}{\textit{Retain 160 Tokens (\(\downarrow\) 94.4\%)}} \\
\midrule
\rowcolor[rgb]{0.88,0.96,0.88}
PruMerge+ (ICCV25) & 57.9 & 70.1 & 52.8 & 72.1 & 1345.9 & 63.2 & 57.1 & 30.6 & 58.9 & 87.6\% \\
\rowcolor[rgb]{0.88,0.96,0.88}
TRIM (COLING25) & 58.9 & 69.1 & 49.2 & \textbf{87.0} & 1392.3 & 65.7 & 51.6 & 27.8 & 59.9 & 89.1\% \\
\rowcolor[rgb]{0.88,0.96,0.88}
VisionZip (CVPR25) & 57.8 & 69.7 & 58.6 & 76.8 & 1393.9 & 64.8 & 60.0 & 35.9 & 61.7 & 91.8\% \\
\rowcolor[rgb]{0.88,0.93,1.0}
DART (EMNLP25) & 58.7 & 70.1 & 57.2 & 75.7 & 1389.3 & 64.6 & 60.8 & 35.0 & 61.4 & 91.4\% \\
\rowcolor[rgb]{0.88,0.93,1.0}
DivPrune (CVPR25) & 60.0 & \textbf{71.4} & 56.3 & 81.9 & 1436.7 & 65.1 & 60.9 & 37.4 & 63.1 & 93.9\% \\
\rowcolor[rgb]{0.88,0.97,0.97}
SCOPE (NeurIPS25) & \textbf{61.2} & \underline{71.2} & \underline{59.2} & 82.7 & \underline{1473.7} & \underline{66.2} & \textbf{62.9} & \underline{37.8} & \underline{64.4} & \underline{95.8\%} \\
\rowcolor[rgb]{0.93,0.88,1.0}
\textbf{TOPS (Ours)} & \underline{61.1} & 70.6 & \textbf{59.7} & \underline{83.9} & \textbf{1480.9} & \textbf{66.6} & \underline{62.8} & \textbf{40.4} & \textbf{64.9} & \textbf{96.6\%} \\
\bottomrule
\end{tabular}
}

%
%
%
    \vspace{-3mm}
\end{table*}

\section{Additional Ablation Studies}
\label{sec:appendix_ablation}

\subsection{Hyperparameter Sensitivity}
\label{sec:appendix_hyperparams}

Table~\ref{tab:ablation_hyperparams} analyzes sensitivity to $\alpha$ (diversity weight) and $\lambda$ (coverage weight). Even a small coverage weight ($\lambda \approx 0.1$) provides consistent improvements, while a moderate diversity weight ($\alpha \approx 0.5$) yields the most stable results.

\begin{figure*}[t]
    \centering
    \includegraphics[width=\linewidth]{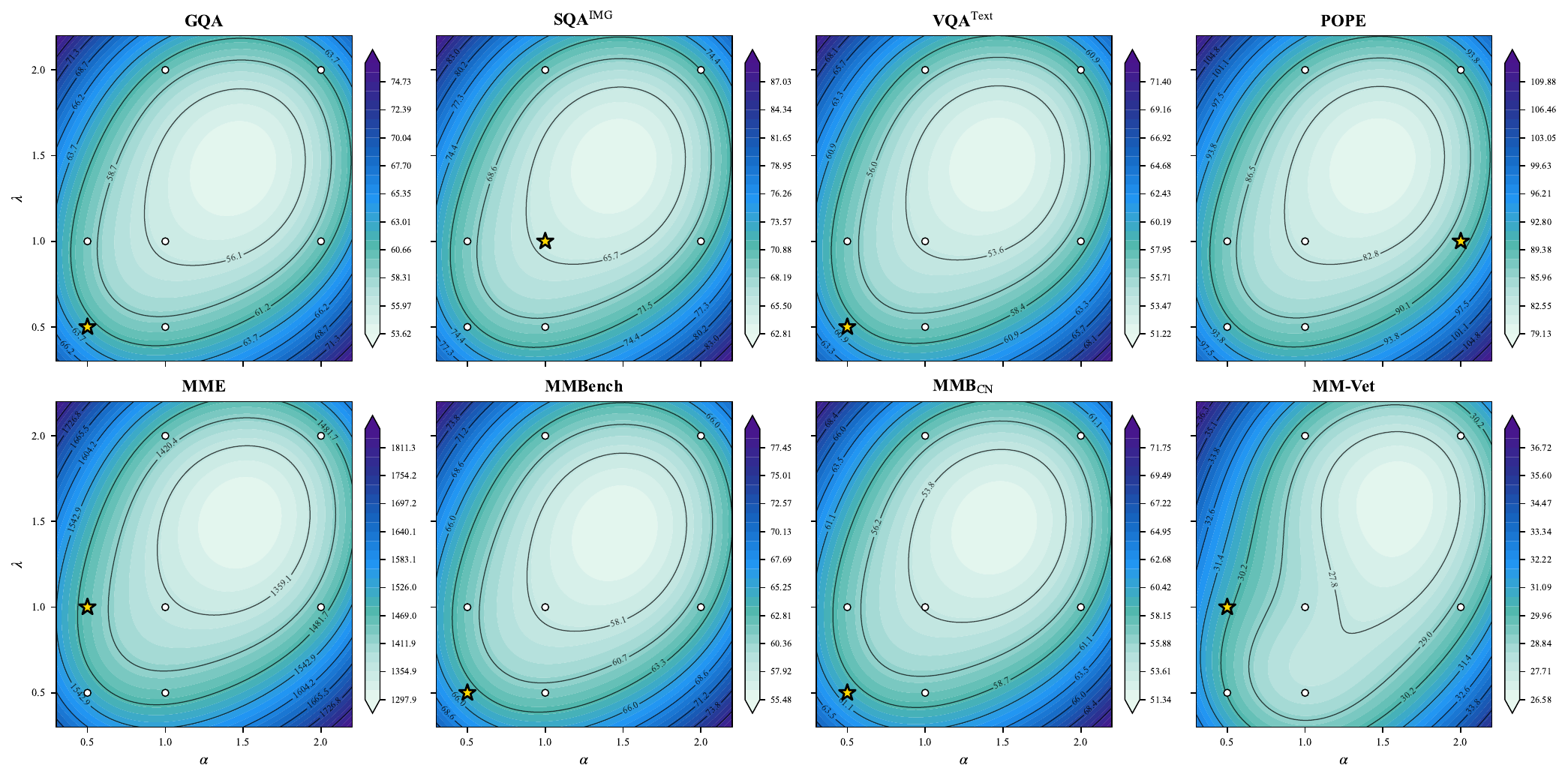}
    \caption{\textbf{Full hyperparameter sensitivity across all 8 benchmarks.} Contour plots of per-benchmark performance across seven $(\alpha, \lambda)$ configurations at 64 tokens on LLaVA-1.5-7B. \textbf{Star}: optimal configuration; \textbf{white dots}: other tested configurations.}
    \label{fig:appendix_hyperparams_full}
\end{figure*}

\begin{table*}[!t]
\centering
\begingroup
\small
\setlength{\tabcolsep}{3pt}
\renewcommand{\arraystretch}{0.90}
\caption{Ablation study of hyper-parameters $\alpha$ (diversity weight) and $\lambda$ (coverage weight) on LLaVA-1.5-7B. \textbf{Acc.} is averaged over benchmarks. \textbf{Bold} and \underline{underline} denote the best and second results per column within each group.}
\label{tab:ablation_hyperparams}
\vspace{-1mm}
\resizebox{0.98\textwidth}{!}{%
\begin{tabular}{cc|cccccccc|c}
\toprule
$\alpha$ & $\lambda$ & \textbf{GQA} & \textbf{SQA}$^{IMG}$ & \textbf{VQA}$^{\text{Text}}$ & \textbf{POPE} & \textbf{MME} & \textbf{MMB}$^{\text{EN}}$ & \textbf{MMB}$^{\text{CN}}$ & \textbf{MMVet} & \textbf{Acc.} \\
\midrule
\multicolumn{11}{c}{\textit{Retain 64 Tokens}} \\
\midrule
0.5 & 0.5 & \textbf{58.9} & \underline{68.6} & \textbf{56.4} & \underline{86.5} & \underline{1441.7} & \textbf{61.3} & \textbf{56.9} & 28.6 & \underline{61.2} \\
1   & 0.5 & 58.7 & 68.3 & \underline{56.3} & 86.2 & 1436.4 & 60.5 & \underline{56.6} & 26.6 & 60.6 \\
0.5 & 1   & 58.7 & \underline{68.6} & 56.2 & \underline{86.5} & \textbf{1442.7} & 60.9 & 56.5 & \textbf{30.6} & \textbf{61.3} \\
1   & 1   & 58.5 & \textbf{68.8} & 56.1 & \underline{86.5} & 1436.8 & 60.7 & 56.4 & \underline{30.5} & \underline{61.2} \\
2   & 1   & 58.6 & \underline{68.6} & 56.2 & \textbf{86.8} & 1428.2 & 60.7 & 56.5 & 29.1 & 61.0 \\
1   & 2   & \textbf{58.9} & \underline{68.6} & 55.8 & 86.4 & 1417.6 & \underline{61.0} & 56.0 & 28.6 & 60.8 \\
2   & 2   & \underline{58.8} & \underline{68.6} & 56.1 & 86.4 & 1390.5 & 60.5 & 55.8 & 27.4 & 60.4 \\
\midrule
\multicolumn{11}{c}{\textit{Retain 32 Tokens}} \\
\midrule
0.5 & 0.5 & 55.9 & \underline{69.1} & \textbf{55.5} & 81.6 & 1362.6 & \underline{60.2} & \textbf{55.5} & \underline{28.9} & 59.3 \\
1   & 0.5 & \underline{56.8} & 68.7 & \underline{55.2} & 83.4 & 1357.0 & \underline{60.2} & \underline{55.2} & 26.6 & 59.2 \\
0.5 & 1   & 56.7 & 68.8 & 55.1 & 83.5 & \underline{1384.7} & 59.5 & 55.1 & \textbf{29.7} & \textbf{59.7} \\
1   & 1   & \textbf{56.9} & \textbf{69.3} & 54.8 & 84.5 & 1361.8 & \underline{60.2} & 54.8 & 26.4 & 59.4 \\
2   & 1   & \underline{56.8} & 68.7 & 54.4 & \textbf{84.8} & 1370.3 & \textbf{60.5} & 53.6 & 27.4 & 59.3 \\
1   & 2   & \underline{56.8} & 68.7 & 54.0 & \underline{84.6} & \textbf{1388.9} & 59.5 & 54.0 & 28.4 & \underline{59.6} \\
2   & 2   & 55.6 & 69.0 & 54.4 & 81.7 & 1327.9 & 59.4 & 54.6 & 27.2 & 58.5 \\
\bottomrule
\end{tabular}}
\vspace{-3mm}
\endgroup
\end{table*}

\subsection{Effect of Dynamic Text Rater}
\label{sec:appendix_text_rater}

Table~\ref{tab:ablation_text_rater} compares different text rater strategies across three pruning ratios. The dynamic rater consistently outperforms \texttt{last\_token} and \texttt{all\_mean} by focusing on text tokens most engaged with visual information at each layer.

\begin{table}[H]
    \centering
    \scriptsize
\setlength{\tabcolsep}{3pt}

\caption{Ablation of text rater strategy.}
\label{tab:ablation_text_rater}
\vspace{-1pt}

\resizebox{\linewidth}{!}{%
\begin{tabular}{l ccc ccc}
\toprule
& \multicolumn{3}{c}{\textbf{Pruning 77.8\%}} 
& \multicolumn{3}{c}{\textbf{Pruning 88.9\%}} \\
\cmidrule(lr){2-4} \cmidrule(lr){5-7}
\textbf{Strategy} & MME & MMBench & GQA & MME & MMBench & GQA \\
\midrule
\texttt{all\_mean}    & \underline{1480.5} & \underline{62.4} & \underline{60.3} & \underline{1441.7} & \textbf{61.2} & \textbf{58.7} \\
\texttt{last\_token}  & \textbf{1488.9}    & \underline{62.4} & 60.2              & 1401.8             & \underline{61.0} & 58.4 \\
\midrule
\textbf{Ours}         & 1482.7             & \textbf{62.5}    & \textbf{60.5}    & \textbf{1442.7}    & 60.9 & \textbf{58.7} \\
\bottomrule
\end{tabular}
}
\end{table}

\subsection{Effect of Pruning Layers}
\label{sec:appendix_pruning_layers}

Table~\ref{tab:ablation_pruning_layers} varies the pruning layer positions. Middle layers $\{12, 24\}$ offer the best balance: early enough for computational savings, yet late enough for sufficient text--visual interaction.

\begin{table*}[!t]
\centering
\begingroup
\small
\setlength{\tabcolsep}{3pt}
\renewcommand{\arraystretch}{0.92}
\caption{Ablation of pruning layer configurations in Stage~2 for LLaVA-1.5-7B. All variants apply Stage~1 ($576{\to}256$) identically. Avg.\ is mean score across seven benchmarks; Rel.\ is relative to unpruned baseline (63.1).}
\label{tab:ablation_pruning_layers}
\vspace{-1mm}
\resizebox{0.98\textwidth}{!}{%
\begin{tabular}{lcccccccccc}
\toprule
\textbf{Stage-2 Layers} & \textbf{MME} & \textbf{MMB} & \textbf{MMB\textsubscript{CN}} & \textbf{SQA} & \textbf{MMVet} & \textbf{TVQA} & \textbf{GQA} & \textbf{POPE} & \textbf{Avg} & \textbf{Rel (\%)} \\
\midrule
L2 ($256{\to}119$)             
& 1421.1 & 60.5 & 54.8 & \textbf{68.4} & \ul{30.3} & \ul{57.0} & 58.9 & 85.9 & 60.9 & 96.5 \\

L10 ($256{\to}70$)             
& \ul{1472.0} & \ul{61.6} & \textbf{57.5} & \textbf{68.4} & 28.7 & 56.9 & \ul{60.0} & 86.0 & \ul{61.6} & \ul{97.6} \\

L2+L14 ($256{\to}128{\to}114$) 
& 1439.1 & 60.9 & 55.4 & 68.0 & \textbf{31.3} & \textbf{57.2} & 59.1 & \ul{86.3} & 61.3 & 97.1 \\

\textbf{L12+L24 (TOPS)} ($256{\to}64{\to}32$) 
& \textbf{1482.7} & \textbf{62.5} & \ul{57.2} & \ul{68.2} & 30.0 & \ul{57.0} & \textbf{60.5} & \textbf{86.8} & \textbf{62.0} & \textbf{98.3} \\
\bottomrule
\end{tabular}}
\vspace{-3mm}
\endgroup
\end{table*}

\section{Additional Empirical Study}
\label{sec:appendix_empirical}

\subsection{Logit Fidelity on an Additional Dataset}
\label{sec:appendix_empirical_logit}

To further validate the empirical observations reported in Section~\ref{sec:principle_analysis}, we extend the logit fidelity analysis to TextVQA. As shown in Figure~\ref{fig:loss_diff_comparison}, we measure $\Delta\mathcal{L} = \mathcal{L}_{\text{pruned}} - \mathcal{L}_{\text{vanilla}}$ across token budgets of 128, 64, and 32 on 200 TextVQA samples. The pattern closely mirrors that observed on MME: at low pruning ratios, relevance-based pruning (FastV) incurs smaller logit distortion; as the budget decreases, diversity and coverage methods exhibit lower degradation. Across all budgets, TOPS consistently achieves the smallest logit increase, confirming that the complementary advantage of combining all three principles generalizes across datasets.

\subsection{Cross-Layer Token Selection Instability}
\label{sec:appendix_cross_layer}

A key motivation for TOPS's multi-stage progressive pruning design is that token importance varies substantially across LLM layers. Figure~\ref{fig:cross_layer_jaccard_all} measures the mean Jaccard similarity between the top-$R{=}128$ token sets selected independently at each pair of layers, computed over 1000 POPE samples on LLaVA-1.5-7B. Near-zero off-diagonal similarities indicate that token selection is highly layer-dependent, motivating progressive multi-stage pruning rather than relying on a single fixed pruning layer.

\subsection{Token Selection Spatial Frequency}
\label{sec:appendix_tops_spatial}

Figure~\ref{fig:spatial_frequency} visualizes the spatial selection frequency heatmaps for representative baselines---FastV, DivPrune, DART, and SCOPE---averaged over 9000 POPE samples at a token budget of 128. FastV exhibits pronounced positional bias toward bottom rows due to attention shift in shallow LLM layers. Figure~\ref{fig:tops_stages_heatmap} further shows the per-token selection probability of TOPS across pruning stages, confirming that TOPS maintains spatially balanced token selection.

\begin{figure*}[p]
\centering
\captionsetup{font=small,skip=3pt}
\setlength{\abovecaptionskip}{3pt}
\setlength{\belowcaptionskip}{5pt}

\vspace*{-4mm}

\includegraphics[width=0.82\textwidth]{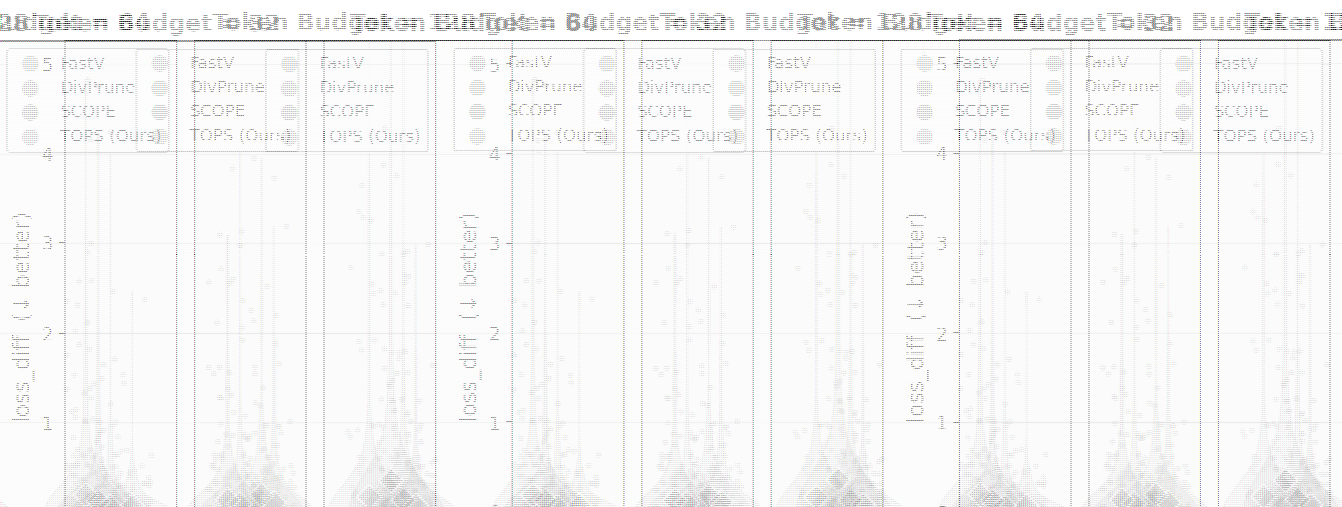}
\caption{Logit fidelity comparison across pruning methods and token budgets on 200 TextVQA samples. TOPS consistently achieves the smallest logit distortion across all budgets.}
\label{fig:loss_diff_comparison}

\vspace{5mm}

\begin{minipage}{0.92\textwidth}
\centering
\includegraphics[width=0.30\linewidth, trim={0 0 0 20pt}, clip]{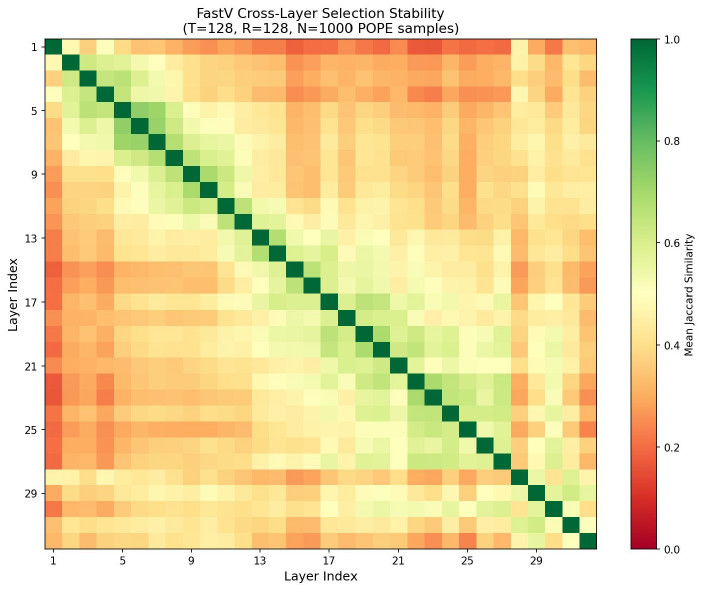}
\hfill
\includegraphics[width=0.30\linewidth, trim={0 0 0 20pt}, clip]{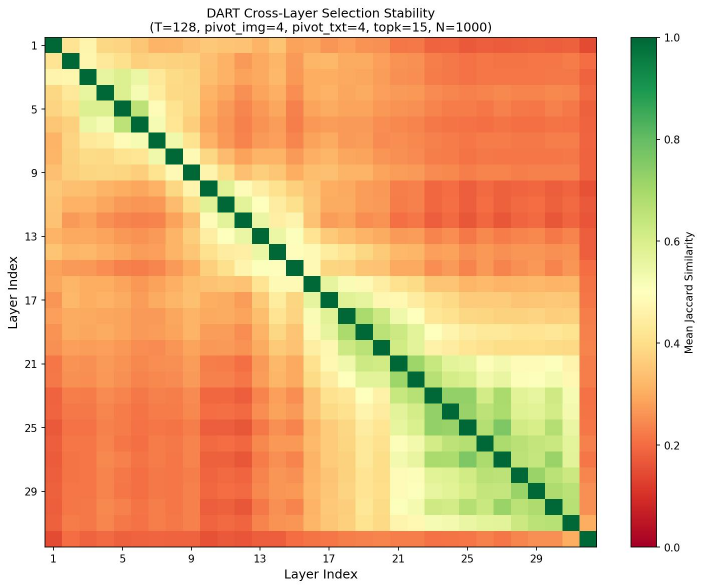}
\hfill
\includegraphics[width=0.30\linewidth, trim={0 0 0 20pt}, clip]{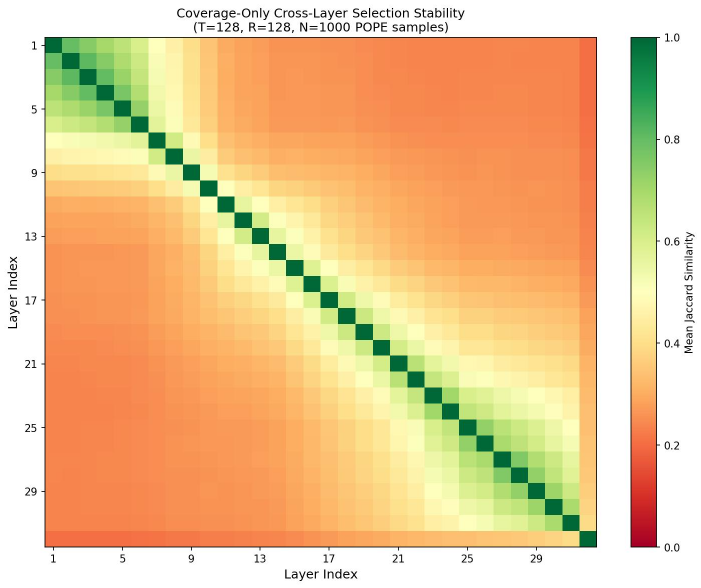}
\end{minipage}
\caption{Cross-layer token selection stability via mean Jaccard similarity ($R{=}128$, $N{=}1000$ POPE samples, LLaVA-1.5-7B). Left to right: attention-based, diversity-based, and coverage-based criteria. Near-zero off-diagonal values confirm that cross-layer inconsistency is universal, justifying TOPS's multi-stage design.}
\label{fig:cross_layer_jaccard_all}

\vspace{5mm}

\includegraphics[width=0.78\textwidth]{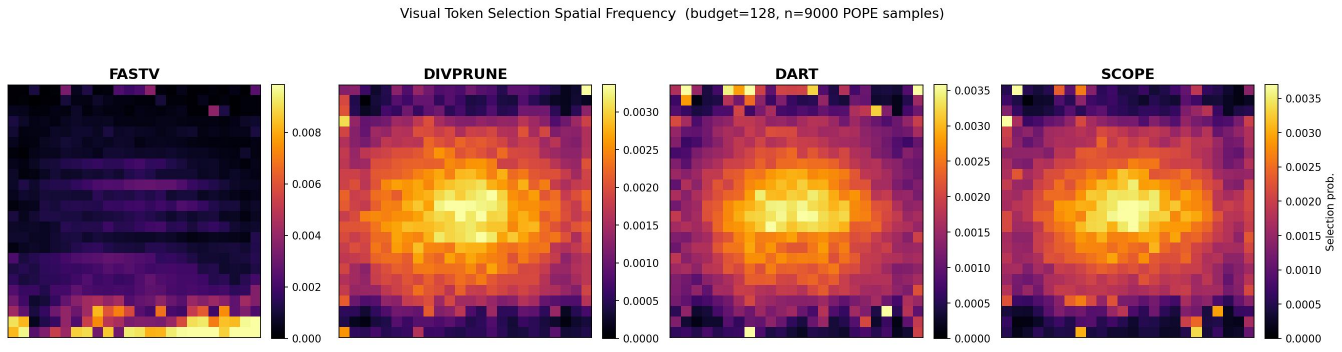}
\caption{Spatial selection frequency heatmaps for FastV, DivPrune, DART, and SCOPE ($9000$ POPE samples, budget$=128$). FastV shows strong positional bias toward bottom rows due to attention shift; other methods achieve roughly uniform spatial coverage.}
\label{fig:spatial_frequency}

\vspace{5mm}

\includegraphics[width=0.78\textwidth]{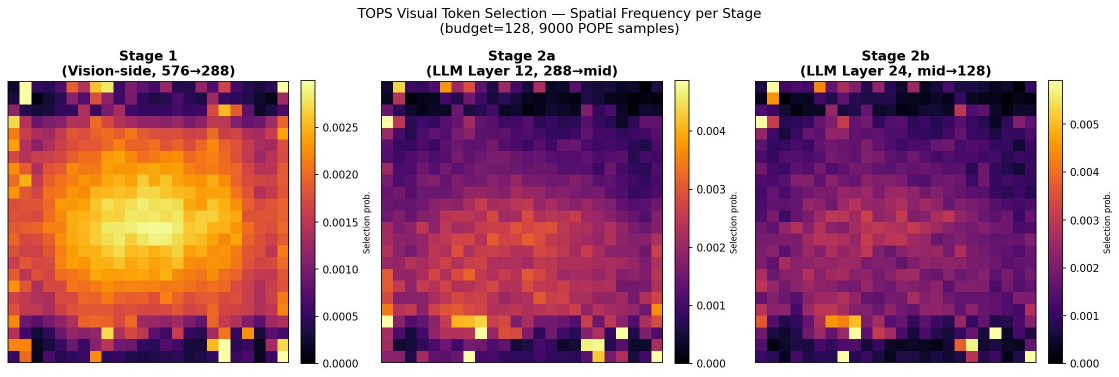}
\caption{Per-token selection probability of TOPS across three pruning stages (budget$=128$, $9000$ POPE samples, LLaVA-1.5-7B). All stages maintain low Gini ($\leq 0.290$) and high normalized entropy ($\geq 0.975$), confirming spatially uniform token selection.}
\label{fig:tops_stages_heatmap}

\vspace{-2mm}
\end{figure*}

\section{Visualization of TOPS}
\label{sec:appendix_visualization}

Figure~\ref{fig:vis_vanilla_comparison} compares TOPS against the Vanilla (unpruned) model. For each example, we show the original image, the token selection heatmap, and the generated answer. Under aggressive compression, TOPS focuses on task-relevant regions---text, key objects, and fine-grained details---while discarding redundant background.

\begin{figure*}[p]
    \centering
    \includegraphics[width=\linewidth]{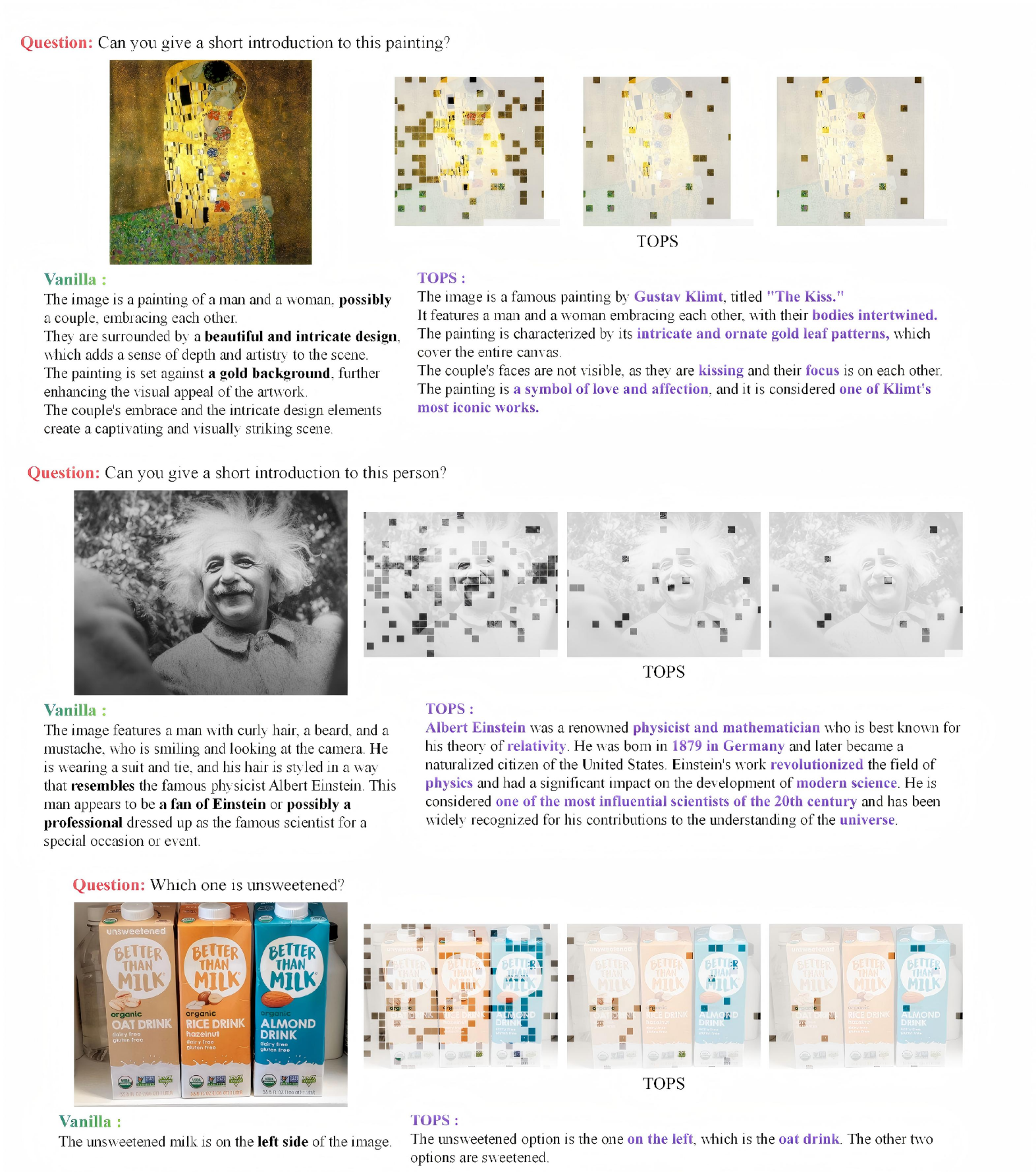}
    \caption{Qualitative comparison of visual token selections between the Vanilla model (no pruning) and TOPS across multiple examples. Despite using far fewer tokens, TOPS selects task-relevant regions and produces correct answers.}
    \label{fig:vis_vanilla_comparison}
\end{figure*}

Figure~\ref{fig:vis_summary_comparison} extends the comparison to multiple baselines (FastV, DivPrune, SCOPE) across diverse questions. Green text denotes correct answers; red denotes incorrect ones.

\begin{figure*}[p]
    \centering
    \includegraphics[width=\linewidth]{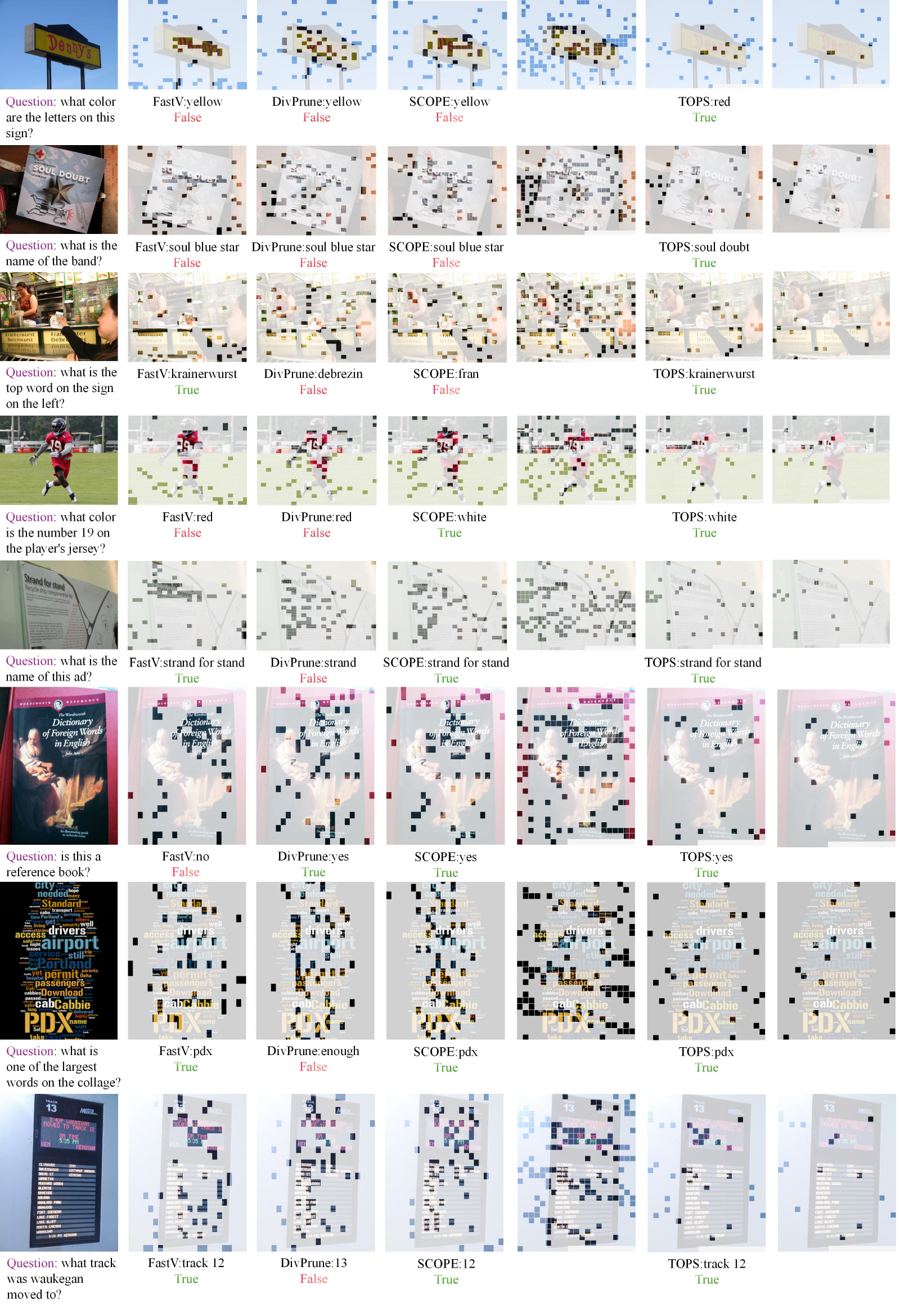}
    \caption{Comprehensive qualitative comparison of visual token selections by FastV, DivPrune, SCOPE, and TOPS across diverse real-world questions. Green text indicates a correct answer; red indicates an incorrect answer.}
    \label{fig:vis_summary_comparison}
\end{figure*}

\section{Per-Benchmark Radar Visualization}
\label{sec:appendix_radar}

To provide a more intuitive view of per-benchmark performance across all compression ratios and model variants, we present radar charts covering LLaVA-1.5 (7B and 13B), LLaVA-NeXT (7B and 13B), Qwen2.5-VL-7B and InternVL3-8B. Each axis corresponds to one benchmark; the outer boundary on each axis is set by the highest-scoring method. \textbf{TOPS} (red) consistently covers the largest area across all settings.

\begin{figure*}[h]
    \centering
    \begin{subfigure}[b]{0.3\linewidth}
        \centering
        \includegraphics[width=\textwidth]{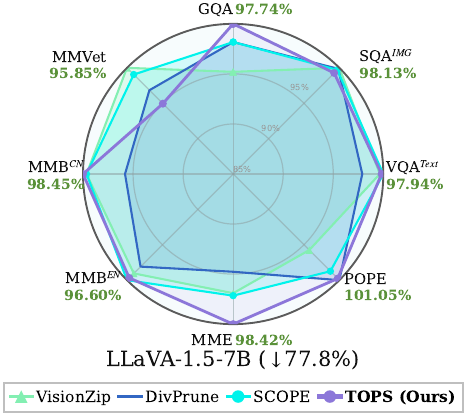}
        \caption{\textbf{LLaVA-1.5-7B} — 128 tokens ($\downarrow$77.8\%)}
        \label{fig:radar_15_7b_128}
    \end{subfigure}
    \hfill
    \begin{subfigure}[b]{0.3\linewidth}
        \centering
        \includegraphics[width=\textwidth]{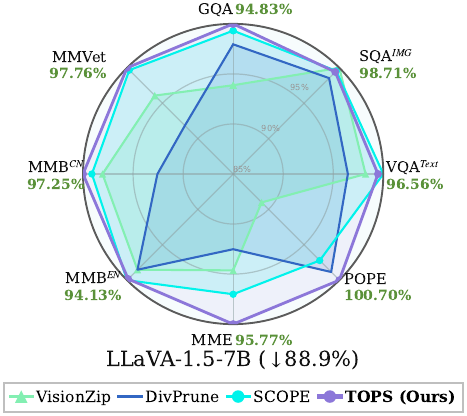}
        \caption{\textbf{LLaVA-1.5-7B} — 64 tokens ($\downarrow$88.9\%)}
        \label{fig:radar_15_7b_64}
    \end{subfigure}
    \hfill
    \begin{subfigure}[b]{0.3\linewidth}
        \centering
        \includegraphics[width=\textwidth]{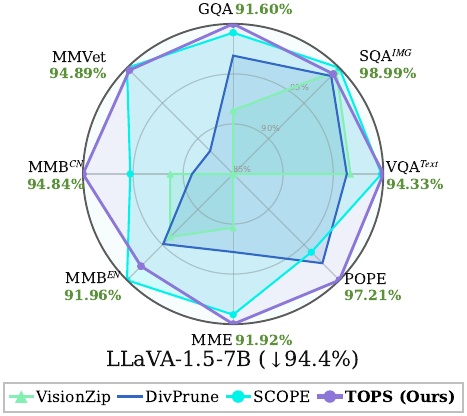}
        \caption{\textbf{LLaVA-1.5-7B} — 32 tokens ($\downarrow$94.4\%)}
        \label{fig:radar_15_7b_32}
    \end{subfigure}
    \caption{Radar charts for LLaVA-1.5-7B at three compression levels.}
    \label{fig:radar_15_7b_all}
\end{figure*}

\begin{figure*}[h]
    \centering
    \begin{subfigure}[b]{0.3\linewidth}
        \centering
        \includegraphics[width=\textwidth]{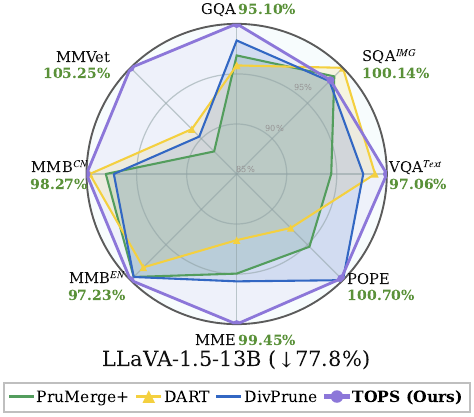}
        \caption{\textbf{LLaVA-1.5-13B} — 128 tokens}
        \label{fig:radar_15_13b_128}
    \end{subfigure}
    \hfill
    \begin{subfigure}[b]{0.3\linewidth}
        \centering
        \includegraphics[width=\textwidth]{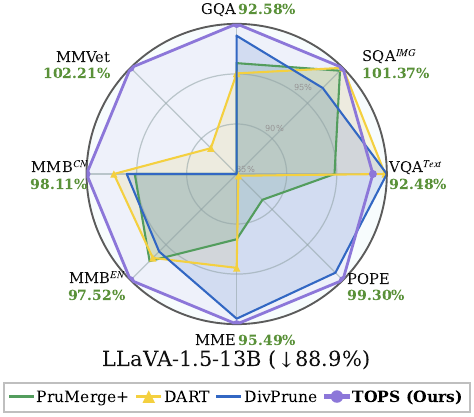}
        \caption{\textbf{LLaVA-1.5-13B} — 64 tokens}
        \label{fig:radar_15_13b_64}
    \end{subfigure}
    \hfill
    \begin{subfigure}[b]{0.3\linewidth}
        \centering
        \includegraphics[width=\textwidth]{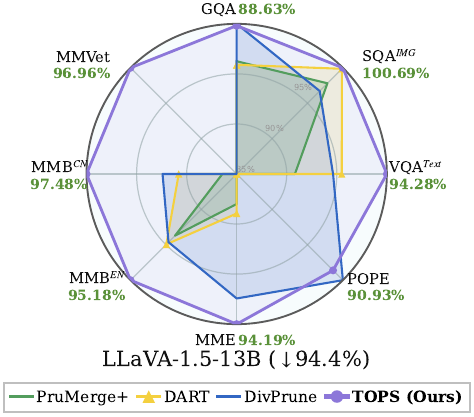}
        \caption{\textbf{LLaVA-1.5-13B} — 32 tokens}
        \label{fig:radar_15_13b_32}
    \end{subfigure}
    \caption{Radar charts for LLaVA-1.5-13B at three compression levels.}
    \label{fig:radar_15_13b_all}
\end{figure*}

\begin{figure*}[h]
    \centering
    \begin{subfigure}[b]{0.3\linewidth}
        \centering
        \includegraphics[width=\textwidth]{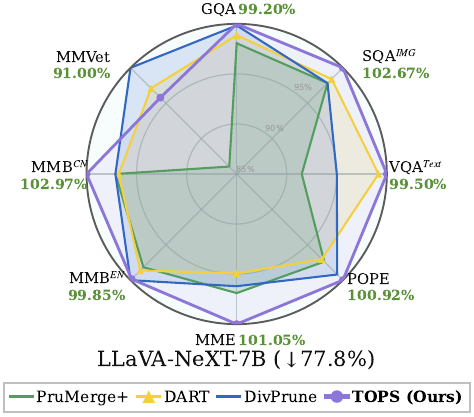}
        \caption{\textbf{LLaVA-NeXT-7B} — 640 tokens}
        \label{fig:radar_next_7b_640}
    \end{subfigure}
    \hfill
    \begin{subfigure}[b]{0.3\linewidth}
        \centering
        \includegraphics[width=\textwidth]{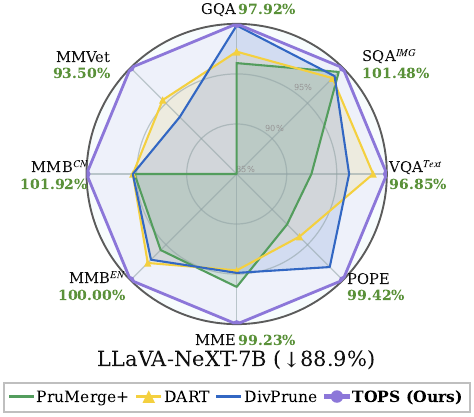}
        \caption{\textbf{LLaVA-NeXT-7B} — 320 tokens}
        \label{fig:radar_next_7b_320}
    \end{subfigure}
    \hfill
    \begin{subfigure}[b]{0.3\linewidth}
        \centering
        \includegraphics[width=\textwidth]{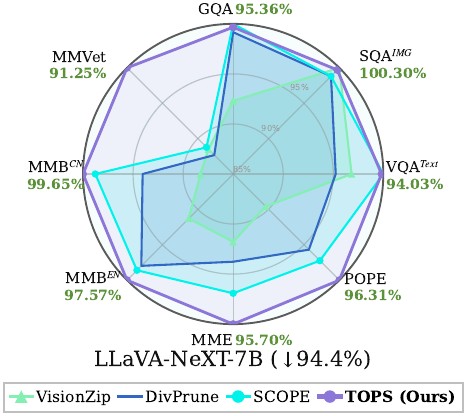}
        \caption{\textbf{LLaVA-NeXT-7B} — 160 tokens}
        \label{fig:radar_next_7b_160}
    \end{subfigure}
    \caption{Radar charts for LLaVA-NeXT-7B at three compression levels.}
    \label{fig:radar_next_7b_all}
\end{figure*}

\begin{figure*}[h]
    \centering
    \begin{subfigure}[b]{0.3\linewidth}
        \centering
        \includegraphics[width=\textwidth]{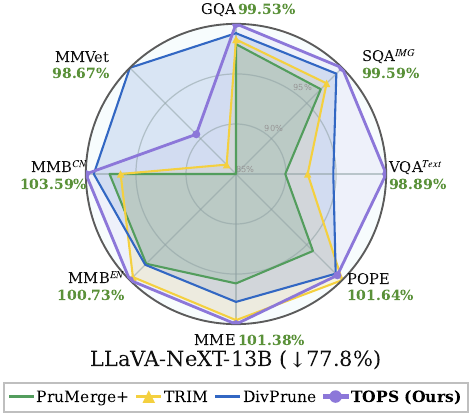}
        \caption{\textbf{LLaVA-NeXT-13B} — 640 tokens}
        \label{fig:radar_next_13b_640}
    \end{subfigure}
    \hfill
    \begin{subfigure}[b]{0.3\linewidth}
        \centering
        \includegraphics[width=\textwidth]{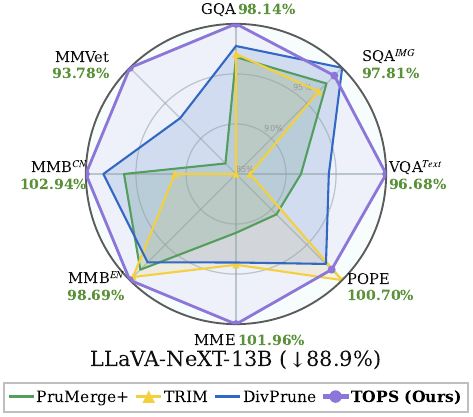}
        \caption{\textbf{LLaVA-NeXT-13B} — 320 tokens}
        \label{fig:radar_next_13b_320}
    \end{subfigure}
    \hfill
    \begin{subfigure}[b]{0.3\linewidth}
        \centering
        \includegraphics[width=\textwidth]{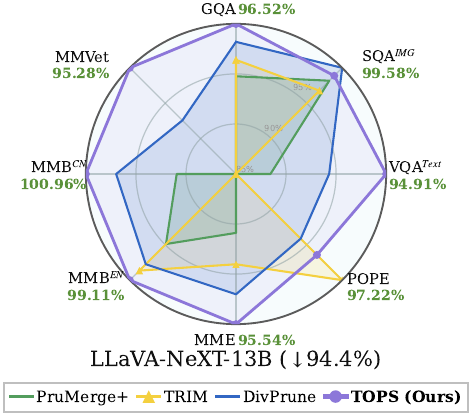}
        \caption{\textbf{LLaVA-NeXT-13B} — 160 tokens}
        \label{fig:radar_next_13b_160}
    \end{subfigure}
    \caption{Radar charts for LLaVA-NeXT-13B at three compression levels.}
    \label{fig:radar_next_13b_all}
\end{figure*}

\begin{figure*}[h]
    \centering

    \begin{subfigure}[b]{0.3\linewidth}
        \centering
        \includegraphics[width=\linewidth]{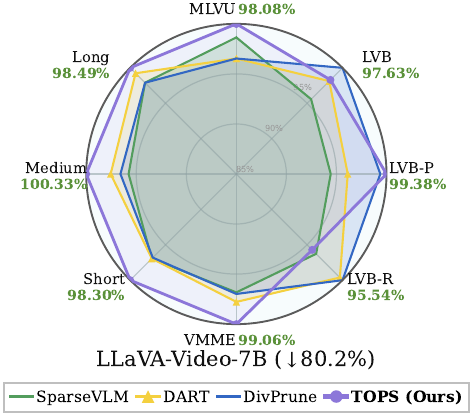}
        \caption{\textbf{LLaVA-Video-7B} — 64 tok/frame}
        \label{fig:radar_video_64}
    \end{subfigure}
    \hfill
    \begin{subfigure}[b]{0.3\linewidth}
        \centering
        \includegraphics[width=\linewidth]{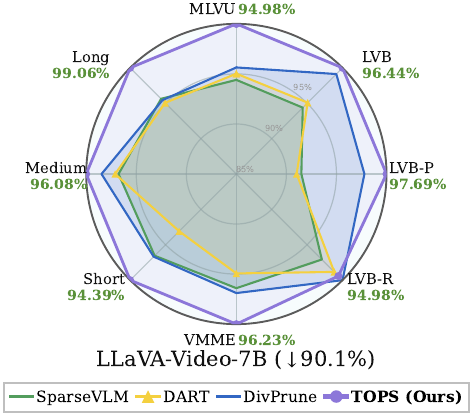}
        \caption{\textbf{LLaVA-Video-7B} — 32 tok/frame}
        \label{fig:radar_video_32}
    \end{subfigure}
    \hfill
    \begin{subfigure}[b]{0.3\linewidth}
        \centering
        \includegraphics[width=\linewidth]{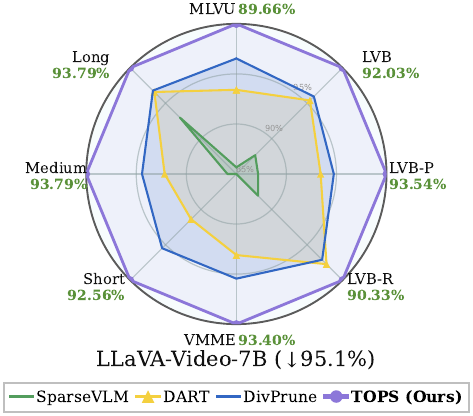}
        \caption{\textbf{LLaVA-Video-7B} — 16 tok/frame}
        \label{fig:radar_video_16}
    \end{subfigure}

    \caption{Radar charts for LLaVA-Video-7B at three compression levels.}
    \label{fig:radar_video_all}

    \vspace{4pt}

    \begin{subfigure}[b]{0.34\linewidth}
        \centering
        \includegraphics[width=\linewidth]{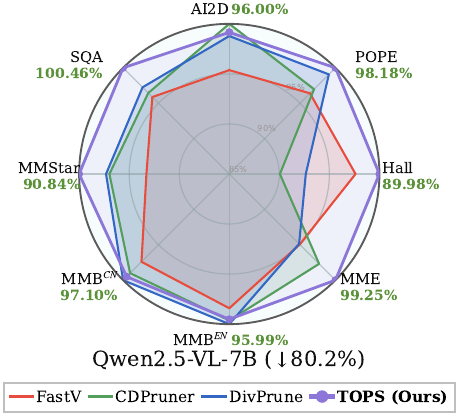}
        \caption{\textbf{Qwen2.5-VL-7B} — 256 tokens}
        \label{fig:radar_qwen_256}
    \end{subfigure}
    \hspace{0.06\linewidth}
    \begin{subfigure}[b]{0.34\linewidth}
        \centering
        \includegraphics[width=\linewidth]{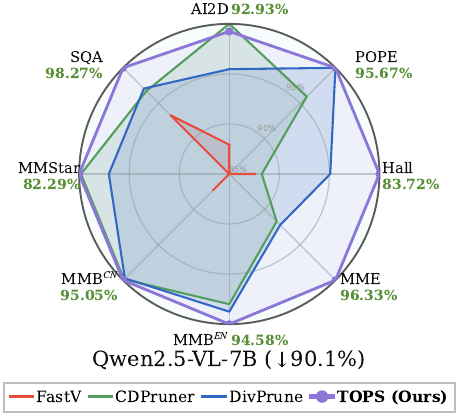}
        \caption{\textbf{Qwen2.5-VL-7B} — 128 tokens}
        \label{fig:radar_qwen_128}
    \end{subfigure}

    \vspace{2pt}

    \begin{subfigure}[b]{0.34\linewidth}
        \centering
        \includegraphics[width=\linewidth]{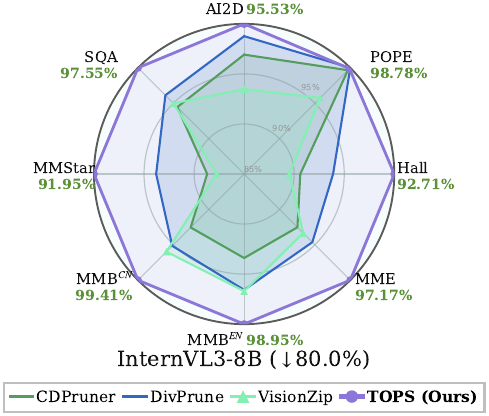}
        \caption{\textbf{InternVL3-8B} — 256 tokens}
        \label{fig:radar_internvl3_256}
    \end{subfigure}
    \hspace{0.06\linewidth}
    \begin{subfigure}[b]{0.34\linewidth}
        \centering
        \includegraphics[width=\linewidth]{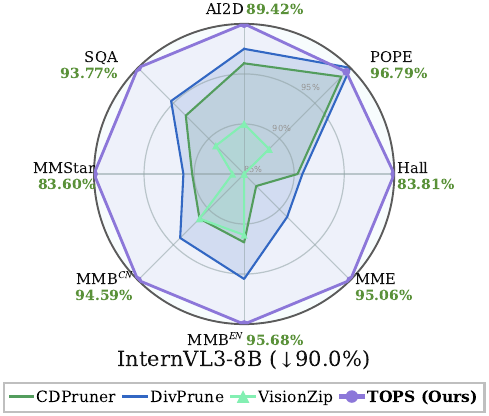}
        \caption{\textbf{InternVL3-8B} — 128 tokens}
        \label{fig:radar_internvl3_128}
    \end{subfigure}

    \caption{Radar charts for Qwen2.5-VL-7B and InternVL3-8B.}
    \label{fig:radar_qwen_internvl_all}

    \vspace{-4mm}
\end{figure*}

\section{Assets, Licenses, and Intended Use}
\label{sec:asset_license}

Our experiments use publicly released models, codebases, and benchmarks solely for non-commercial academic research and evaluation. We build on official open-source codebases and toolkits, including LLaVA, LLaVA-NeXT, \texttt{lmms-eval}, and \texttt{VLMEvalKit}, and follow the licenses and usage terms of the corresponding model providers for LLaVA-1.5, LLaVA-NeXT, LLaVA-Video, InternVL3-8B, and Qwen2.5-VL-7B-Instruct. We evaluate on public multimodal benchmarks, some of which impose non-commercial or academic-only restrictions, such as ScienceQA, MLVU, LongVideoBench, MM-Vet, and Video-MME. We do not redistribute third-party model weights, datasets, annotations, or videos; users should obtain these assets from their official sources and comply with their original licenses. We do not collect new human-subject data. Since our experiments use publicly released benchmarks, we rely on their original curation procedures and manually inspect examples used in qualitative visualizations to avoid displaying personally identifying or offensive content.

\section{Broader Impact}
\label{sec:broader_impact}

This work presents TOPS, a training-free visual token pruning method for efficient MLLM inference. TOPS lowers energy consumption and hardware requirements of deploying MLLMs, democratizing access to capable vision-language models on resource-constrained devices and reducing carbon emissions per query. As a training-free, architecture-agnostic method, TOPS does not introduce new biases through retraining. Users should be aware that token pruning may degrade accuracy on inputs requiring fine-grained spatial reasoning under aggressive compression, and should validate pruning configurations on their target task before deployment in safety-critical settings.

\end{document}